\documentclass{article}
\usepackage[final, nonatbib]{neurips_2019}
\usepackage[utf8]{inputenc} % allow utf-8 input
\usepackage[T1]{fontenc}    % use 8-bit T1 fonts
\usepackage{hyperref}       % hyperlinks
\usepackage{url}            % simple URL typesetting
\usepackage{amsfonts}       % blackboard math symbols
\usepackage{nicefrac}       % compact symbols for 1/2, etc.
\usepackage{microtype}      % microtypography
\usepackage[caption=false,font=footnotesize]{subfig}
\usepackage{graphicx}
\usepackage{mathtools}
\usepackage{epstopdf}
\usepackage{pifont}% http://ctan.org/pkg/pifont
\usepackage{pxfonts}
\usepackage{amsmath,amssymb,bm}
\usepackage{footnote}
\usepackage{enumerate}
\usepackage{physics}
\usepackage{enumitem}
\usepackage[english]{babel}
\usepackage{comment}
\usepackage{setspace}
\usepackage{algorithm}
\usepackage{algpseudocode}

\usepackage{forloop}
\usepackage{chngpage}
\usepackage{color, colortbl}
\definecolor{darkgreen}{rgb}{0, 0.5, 0}
\definecolor{red}{rgb}{1, 0, 0}
\definecolor{purple}{rgb}{0.5, 0, 0.5}
\usepackage{booktabs, multirow}

\newcommand\ie{\textit{i.e.}}
\newcommand\eg{\textit{e.g.}}
\newcommand\st{\textit{s.t.}}

\newcommand\etc{\textit{etc.}}

\newcommand\doubleP{\mathbb{P}}

\newcommand{\colvec}[2][1]{%
	\scalebox{#1}{%
		\renewcommand{\arraystretch}{1.45}%
		$\begin{bmatrix}#2\end{bmatrix}$%
	}
}
\makeatletter
\newcommand*{\rom}[1]{\expandafter\@slowromancap\romannumeral #1@}
\makeatother

\newenvironment{theorem}[2][Theorem]{\begin{trivlist}
		\item[\hskip \labelsep {\bfseries #1}\hskip \labelsep {\bfseries #2.}]}{\end{trivlist}}
\newenvironment{lemma}[2][Lemma]{\begin{trivlist}
		\item[\hskip \labelsep {\bfseries #1}\hskip \labelsep {\bfseries #2.}]}{\end{trivlist}}

\newenvironment{proof}{{\noindent\it Proof}\quad}{\hfill $\square$\par}

\def\coef_vec{
	\begin{bmatrix}
		\frac{g^{(0) } (0)}{0!} \\[6pt]
		\frac{g^{(1) } (0)}{1!} \\[6pt]
		\frac{g^{(2) } (0)}{2!}\\[6pt]
		\vdots\\[6pt]
		\frac{g^{(\infty) } (0)}{\infty!}
\end{bmatrix}}

\makeatletter
\newsavebox{\@brx}
\newcommand{\llangle}[1][]{\savebox{\@brx}{\(\m@th{#1\langle}\)}%
	\mathopen{\copy\@brx\kern-0.5\wd\@brx\usebox{\@brx}}}
\newcommand{\rrangle}[1][]{\savebox{\@brx}{\(\m@th{#1\rangle}\)}%
	\mathclose{\copy\@brx\kern-0.5\wd\@brx\usebox{\@brx}}}
\makeatother

\newtheorem{definition}{Definition}

\title{Break the Ceiling: Stronger Multi-scale Deep Graph Convolutional Networks}

\author{
Sitao Luan$^{1,2,*}$, Mingde Zhao$^{1,2,*}$, Xiao-Wen Chang$^{1}$, Doina Precup$^{1,2,3}$\\
\{sitao.luan@mail, mingde.zhao@mail, chang@cs, dprecup@cs\}.mcgill.ca\\
$^1$McGill University; $^2$Mila; $^3$DeepMind\\
$^{*}$Equal Contribution
}

\begin{document}
\maketitle	

\begin{abstract}
Recently, neural network based approaches have achieved significant performance improvement for solving large, complex, graph-structured problems. However, the advantages of multi-scale information and deep architectures have not been sufficiently exploited. In this paper, we analyze how existing Graph Convolutional Networks (GCNs) have limited expressive power due to the constraint of the activation functions and their architectures. We generalize spectral graph convolution and deep GCN in block Krylov subspace forms and devise two architectures, both with the potential to be scaled deeper but each making use of the multi-scale information differently. On several node classification tasks, with or without validation set, the two proposed architectures achieve state-of-the-art performance.
\end{abstract}
	
\section{Introduction and Motivation}\label{sec:introduction}
Many real-world problems can be modeled as graphs \cite{hamilton2017inductive, kipf2016classification, liao2019lanczos, gilmer2017neural, monti2017geometric, defferrard2016fast}. Among the recent focus of applying machine learning algorithms on these problems, graph convolution in Graph Convolutional Networks (GCNs) stands out as one of the most powerful tools and the key operation, which was inspired by the success of Convolutional Neural Networks (CNNs) \cite{lecun1998gradient} in computer vision \cite{li2018adaptive}. In this paper, we focus on spectrum-free Graph Convolutional Networks (GCNs) \cite{bronstein2016geometric, shuman2012emerging}, which have obtained state-of-the-art performance on multiple transductive and inductive learning tasks \cite{defferrard2016fast, kipf2016classification, liao2019lanczos, chen2018fastgcn, chen2017stochastic}.
%There are two main categories of graph convolution: spectral convolution and spatial convolution\cite{zhou2018graph, wu2019survey}. And there are two basic methods to define spectral convolution: spectral methods \cite{bruna2014spectral} and spectrum-free methods\cite{bronstein2016geometric, shuman2012emerging}. Spectrum-free methods become popular these years and have obtained state-of-the-arts performance on multiple transductive learning and inductive learning tasks \cite{defferrard2016fast, kipf2016classification, liao2019lanczos, chen2018fastgcn, chen2017stochastic}.
\par
One big challenge for the existing GCNs is the limited expressive power of their shallow learning mechanisms \cite{zhang2018graph, wu2019survey}. The difficulty of extending GCNs to richer architectures leads to several possible explanations and even some opinions that express the unnecessities of addressing such a problem:
\begin{enumerate}[leftmargin=12pt]
\item Graph convolution can be considered as a special form of Laplacian smoothing \cite{li2018deeper}. A network with multiple convolutional layers will suffer from an over-smoothing problem that makes the representation of the nodes indistinguishable even for the nodes that are far from each other \cite{zhang2018graph}. %{(\bf Do the two sentences have some connections?)}
\item For many cases, it is not necessary for the label information to totally traverse the entire graph. Moreover, one can operate on the multi-scale coarsening of the input graph and obtain the same flow of information as GCNs with more layers \cite{bronstein2016geometric}.
\end{enumerate}
\par
Nevertheless, shallow learning mechanisms violate the compositionality principle of deep learning \cite{lecun2015deep, hinton2006fast} and restrict label propagation \cite{sun2019stage}. In this paper, we first give analyses of the lack of scalability of the existing GCN. Then we show that any graph convolution with a well-defined analytic spectral filter can be written as a product of a block Krylov matrix and a learnable parameter matrix in a special form. Based on the analyses, we propose two GCN architectures that leverage multi-scale information with differently and are scalable to deeper and richer structures, with the expectation of having stronger expressive powers and abilities to extract richer representations of graph-structured data. We also show that the equivalence of the two architectures can be achieved under certain conditions. For validation, we test the proposed architectures on multiple transductive tasks using their different instances. The results show that even the simplest instantiation of the proposed architectures yields state-of-the-art performance and the complex ones achieve surprisingly higher performance, both with or without the validation set.

\section{Preliminaries}
\label{sec:preliminaries}
We use bold font for vectors $\bm{v}$, block vectors $\bm{V}$ and matrix blocks $\bm{V_i}$ as in \cite{frommer2017block}. Suppose we have an undirected graph $\mathcal{G}=(\mathcal{V},\mathcal{E}, A)$, where $\mathcal{V}$ is the node set with $\abs{\mathcal{V}}=N$, $\mathcal{E}$ is the edge set with $\abs{\mathcal{E}}=E$, and $A \in \mathbb{R}^{N\times N}$ is a symmetric adjacency matrix. Let $D$ denote the diagonal degree matrix, \ie{} $D_{ii} = \sum_j A_{ij}$. A diffusion process on $\mathcal{G}$ can be defined by a diffusion operator $L$ \cite{coifman2006diffusion, coifman2006diffusionmaps} which is a symmetric positive semi-definite matrix, \eg{} graph Laplacian $L=D-A$, normalized graph Laplacian $L=I-D^{-1/2} A D^{-1/2}$ and affinity matrix $L = A + I$, \etc{} We use $L$ to denote a general diffusion operator in this paper. The eigendecomposition of $L$ gives us $L=U \Lambda U^T$, where $\Lambda$ is a diagonal matrix whose diagonal elements are eigenvalues and the columns of $U$ are the orthonormal eigenvectors and named graph Fourier basis. We also have a feature matrix (graph signals, can be regarded as a block vector) $\bm{X} \in \mathbb{R}^{N\times F}$ defined on $\mathcal{V}$ and each node $i$ has a feature vector $\bm{X_{i,:}}$, which is the $i^{th}$ row of $X$.

Graph convolution is defined in graph Fourier domain \st{} $\bm{x} *_{\mathcal{G}} \bm{y} = U((U^T \bm{x}) \odot (U^T\bm{y}))$, where $\bm{x}, \bm{y} \in \mathbb{R}^N$ and $\odot$ is the Hadamard product \cite{defferrard2016fast}. Following from this definition, a graph signal $\bm{x}$ filtered by $g_\theta$ can be written as
\begin{equation}\label{def}
    \bm{y} = g_\theta(L)\bm{x} = g_\theta(U \Lambda U^T) \bm{x} = U g_\theta(\Lambda) U^T \bm{x}
\end{equation}
where $g_\theta$ can be any function which is analytic inside a closed contour which encircles $\lambda(L)$, \eg{} Chebyshev polynomial \cite{defferrard2016fast}. GCN generalizes this definition to signals with $F$ input channels and $O$ output channels and the network structure is
\begin{equation}
    \label{eq0}
   \bm{Y} = \text{softmax} ({L} \; \text{ReLU} ( L \bm{X} W_0 ) \; W_1 )
\end{equation}
where $L = D^{-1/2} \tilde{A} D^{-1/2}$ and $\tilde{A} = A+I$. This is called spectrum-free method \cite{bronstein2016geometric} that requires no explicit computation of eigendecomposition \cite{zhang2018graph} and operations on the frequency domain. We will focus on the analysis of GCN in the following sections.

%We do not distinguish vector $v \in \mathbb{R}^{s}$ and block vector $\bm{V} \in \mathbb{R}^{s_1 \times s_2}$ in this section.

%%%So we can write all the features into a matrix $\bm{X}=[\bm{x_1}, \bm{x_2},\dots,\bm{x_F}]\in R^{N\times F}$ and each $\bm{x_i}  (i=1,2,\dots,F)$ is a feature variable. $y\in R^{N\times 1}$ is the label vector and only a small subset of it is available. We have a matrix $L\in R^{N\times N}$ that contains the information from $\mathcal{E}$ and it can represent the relational structures of each pair of nodes. $L$ can be graph Laplacian, affinity matrix of any other kind of similarity matrix or diffusion matrix (operator) that defined on graph. Here, we use graph Laplacian for our analysis.

%%%In GCN, the graph convolution is defined as
%%%%We combine each convolved column by learnable weight matrix $W\in R^{F\times l_1}$ and put a nonlinear activation function on it, then we send the output to the next layer. The input of the next layer is a $N\times l_1$ matrix.
\section{Why GCN is not Scalable?}
Suppose we scale GCN to a deeper architecture in the same way as \cite{kipf2016classification, li2018deeper}, it becomes
\begin{equation}\label{eq1}
\bm{Y} = \text{softmax} ({L} \; \text{ReLU} ( \cdots L \; \text{ReLU} (L\; \text{ReLU} (L \bm{X} W_0 ) \; W_1 )\; W_2 \cdots ) \; W_n ) \equiv  \text{softmax} (\bm{Y'})
\end{equation}
For this, we have the following theorems.
\begin{theorem} 1 \label{thm1}
Suppose that $\mathcal{G}$ has $k$ connected components and the diffusion operator $L$ is defined as that in \eqref{eq0}. Let $\bm{X}$ be any block vector sampled from space $\mathbb{R}^{N \times F}$ according to a continuous distribution and $\{W_0, W_1, \dots, W_n\}$ be any set of parameter matrices, if $\mathcal{G}$ has no bipartite components, then in \eqref{eq1}, as $n \to \infty$, $\text{rank}(\bm{Y'}) \leq k$ almost surely.
%, consider the space $\mathbb{R}^{N \times F}$ of all block vectors $\bm{X}$ and any set of parameter matrices $\{W_0, W_1, \dots, W_n\}$, we have $\text{rank}(\bm{Y'}) \leq k$ with probability 1, as $n\rightarrow \infty$.
\end{theorem}	
\begin{proof}
See Appendix.
\end{proof}

\begin{theorem} 2
\label{thm2}
Suppose we randomly sample $\bm{x}, \bm{y} \in \mathbb{R}^N$ under a continuous distribution and point-wise function $\text{Tanh}(z) = \frac{e^z - e^{-z}}{e^z + e^{-z}}$, we have
$$\doubleP(\text{rank}\left(\text{Tanh}([\bm{x},\bm{y}])\right) \geq \text{rank}([\bm{x},\bm{y}]) \;|\; \bm{x},\bm{y} \in \mathbb{R}^N) = 1$$
\end{theorem}

\begin{proof}
See Appendix.
\end{proof}

Theorem 1 shows that if we simply increase the depth based on GCN architecture, the extracted features $\bm{Y'}$ will at most encode stationary information of graph structure and lose all the information in node features. In addition, from the proof we see that the point-wise ReLU transformation is a conspirator. Theorem 2 tells us that Tanh is better in keeping linear independence among column features. We design a numerical experiment on synthetic data (see Appendix) to test, under a 100-layer GCN architecture, how activation functions affect the rank of the output in each hidden layer during the feed-forward process. As Figure 1(a) shows, the rank of hidden features decreases rapidly with ReLU, while having little fluctuation under Tanh, and even the identity function performs better than ReLU (see Appendix for more comparisons). So we propose to replace ReLU by Tanh.%ReLU is also found to have numerical stability issue in the bottom layers of GCN.
\begin{figure*}[htbp]
\centering
\subfloat[Deep GCN]{
\captionsetup{justification = centering}
\includegraphics[width=0.32\textwidth]{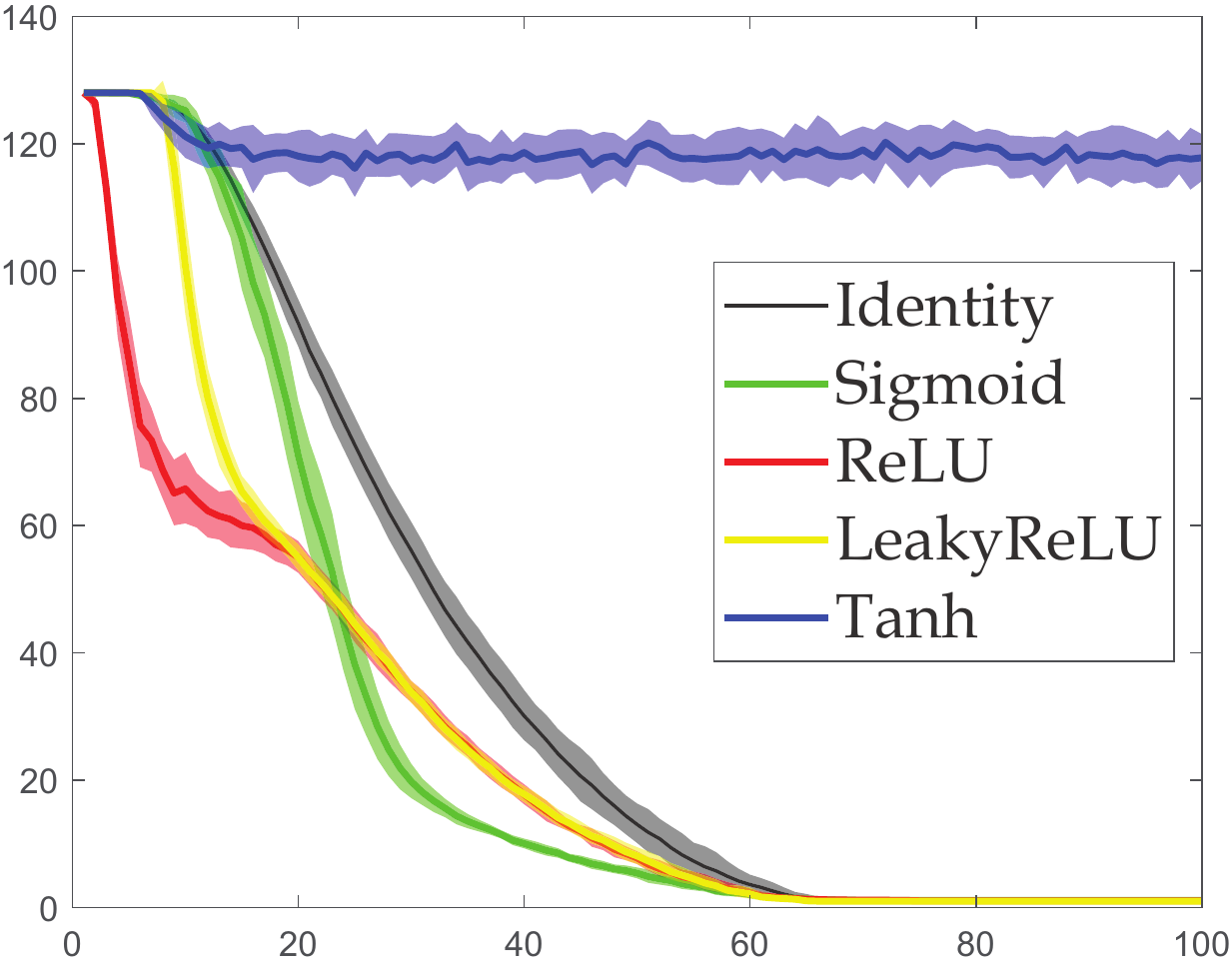}}
\hfill
\subfloat[Snowball]{
\captionsetup{justification = centering}
\includegraphics[width=0.32\textwidth]{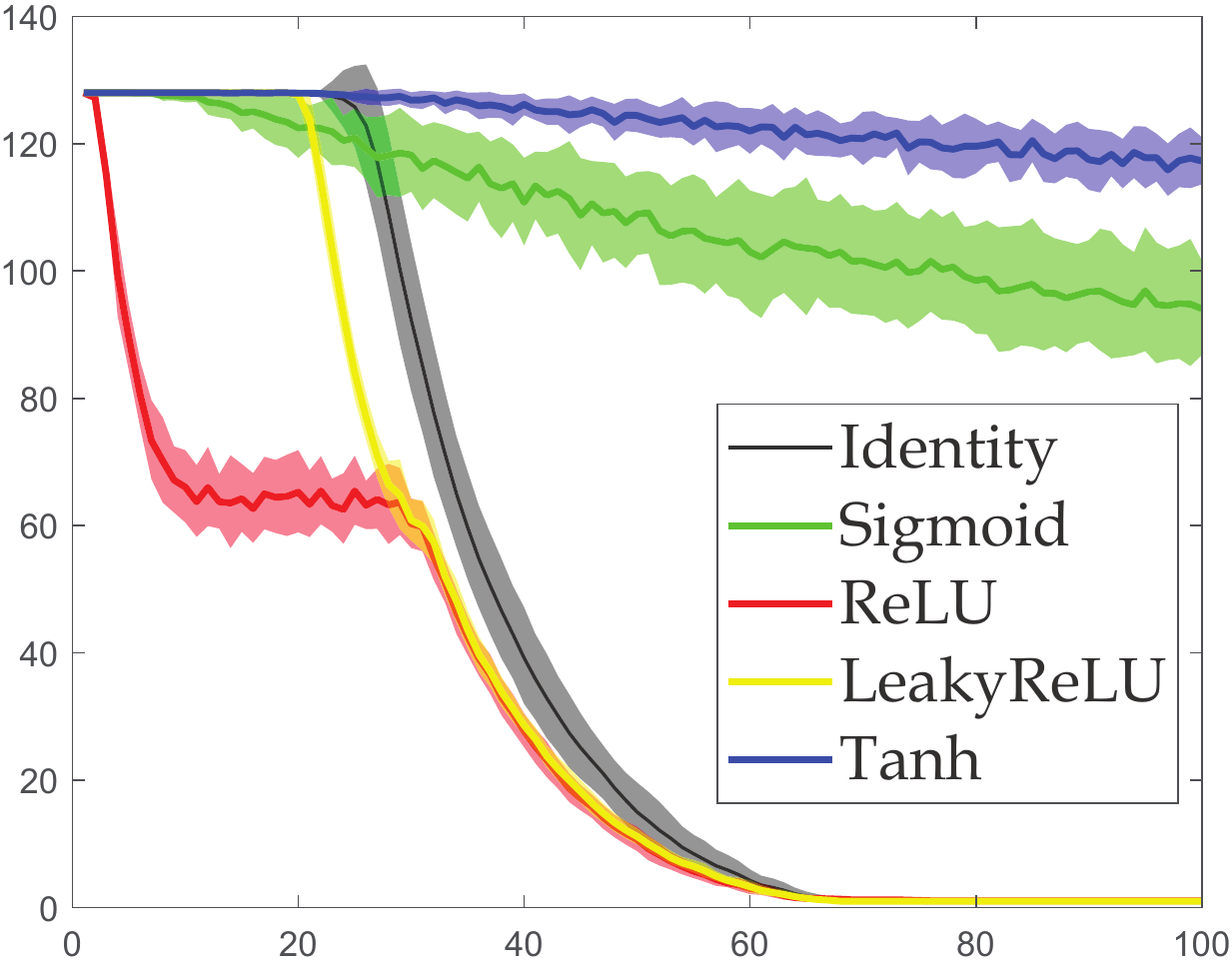}}
\hfill
\subfloat[Truncated Block Krylov]{
\captionsetup{justification = centering}
\includegraphics[width=0.32\textwidth]{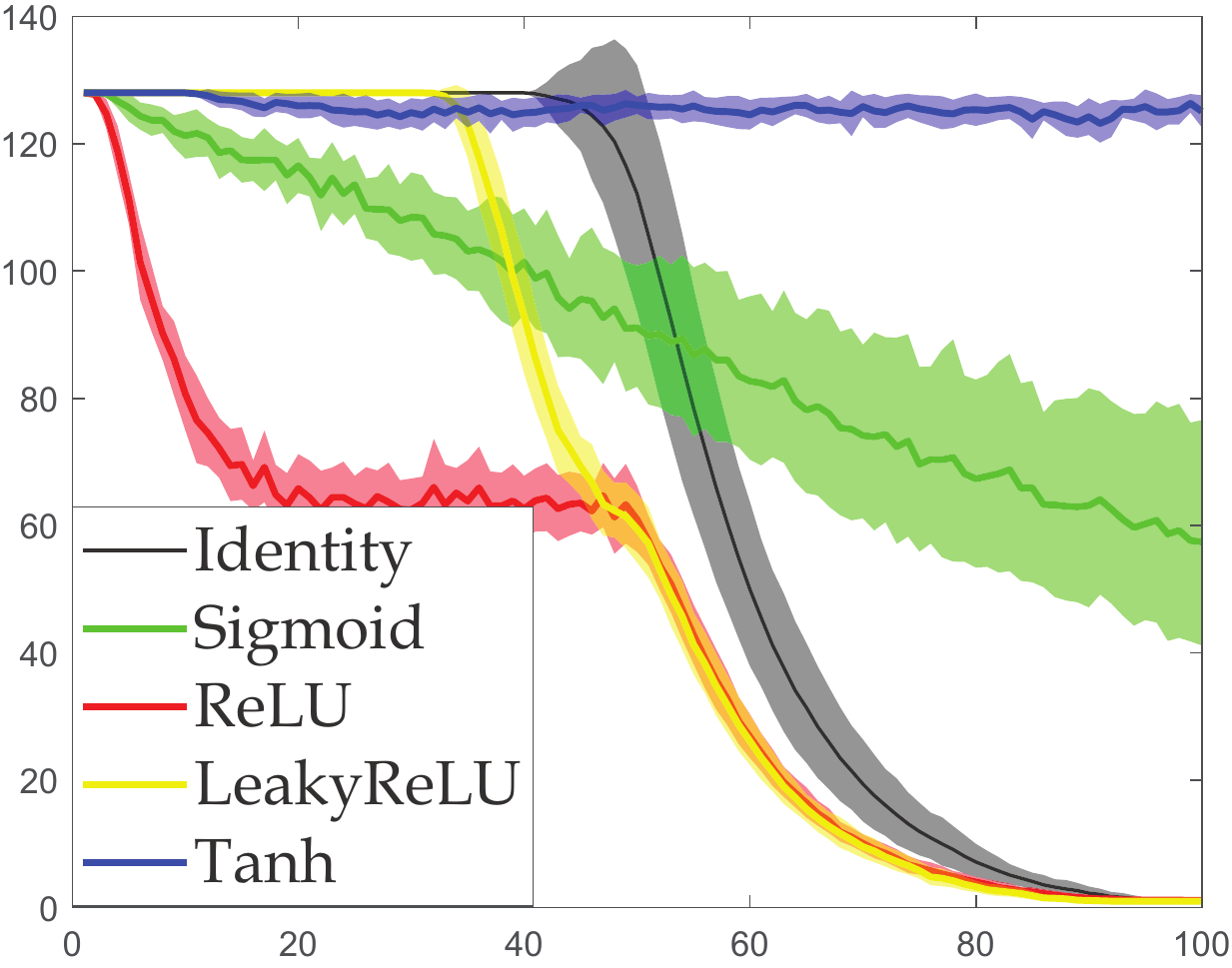}}
\caption{Number of independent column features }
\label{activation_functions}
\end{figure*}
\section{Spectral Graph Convolution and Block Krylov Subspace Methods}
\subsection{Notation and Backgrounds}
%Let $\mathbb{S}$ be a $^*$-subalgebra  with identity $I_s$, \ie{}
Let $\mathbb{S}$ be a block vector subspace of $\mathbb{R}^{F\times F}$ containing the identity matrix $I_F$ that is closed under matrix multiplication and transposition. We define an inner product $\langle\cdot, \cdot\rangle_{\mathbb{S}}$ in the block vector space $\mathbb{R}^{N \times F}$ as follows \cite{frommer2017block},
\begin{definition}
A mapping $\langle\cdot, \cdot\rangle_{\mathbb{S}}$ from $\mathbb{R}^{N\times F} \times \mathbb{R}^{N\times F}$ to $\mathbb{S} $ is called a block inner product onto $\mathbb{S}$ if it satisfies the following conditions for all $\bm{X, Y, Z} \in \mathbb{R}^{N\times F}$ and $C \in \mathbb{S}$:
\begin{enumerate}[leftmargin=12pt]
\item $\mathbb{S}$-linearity: $\langle \bm{X}, \bm{Y}C \rangle_{\mathbb{S}} =  \langle \bm{X}, \bm{Y}\rangle_{\mathbb{S}}C$  and $ \langle \bm{X} + \bm{Y}, \bm{Z} \rangle_{\mathbb{S}} = \langle\bm{X}, \bm{Z}\rangle_{\mathbb{S}} + \langle \bm{Y}, \bm{Z}\rangle_{\mathbb{S}}$
\item symmetry: $ \langle \bm{X}, \bm{Y}\rangle_{\mathbb{S}} = \langle\bm{Y}, \bm{X}\rangle_{\mathbb{S}}^T $
\item definiteness: $ \langle \bm{X}, \bm{X}\rangle_{\mathbb{S}} $ is positive definite if $\bm{X}$ has full rank, and $ \langle \bm{X}, \bm{X}\rangle_{\mathbb{S}} = 0_F$ if and only if $\bm{X} = 0.$
\end{enumerate}
\end{definition}
%\begin{definition}
%A mapping $N$ which maps all $\bm{X} \in \mathbb{R}^{n\times s}$ with full rank to a matrix $N(\bm{X}) \in \mathbb{S}$ is called a scaling quotient if for all such $\bm{X}$ there exists $\bm{Y}\in \mathbb{R}^{n\times s}$ such that $\bm{X}= \bm{Y} N(\bm{X})$ and $ \langle\bm{Y}, \bm{Y}\rangle_{\mathbb{S}} = I_s$ .
%\end{definition}
	
%Two block vectors $\bm{X}, \bm{Y}$ are called $\langle\cdot, \cdot\rangle_{\mathbb{S}} $-orthogonal if $\langle \bm{X}, \bm{Y}\rangle_{\mathbb{S}} = 0_s$ and we call a block vector $\langle\cdot, \cdot\rangle_{\mathbb{S}} $-normalized if $ \langle \bm{X}, \bm{X}\rangle_{\mathbb{S}} = I_s$
There are mainly three ways to define $\langle\cdot, \cdot\rangle_{\mathbb{S}}$, we use the classical one: $\mathbb{S}^{Cl} = \mathbb{R}^{F\times F}$  and $ \langle \bm{X}, \bm{Y}\rangle_{\mathbb{S}}^{Cl}= \bm{X}^{T} \bm{Y}$. %(2) Global: $\mathbb{S}^{Gl} = c I_F,\; c \in \mathbb{R} $   and $ \langle \bm{X}, \bm{Y}\rangle_{\mathbb{S}}^{Gl}= \text{trace}(\bm{X}^T \bm{Y}) I_F $. (3) Loop-interchange: $\mathbb{S}^{Li}$ is the set of diagonal matrices in $\mathbb{S}$  and $ \langle \bm{X}, \bm{Y}\rangle_{\mathbb{S}}^{Li}= \text{diag}(\bm{X}^T \bm{Y}) $.
We define a block vector subspace of $\mathbb{R}^{N\times F}$, which will be used later.
\begin{definition}
Given a set of block vectors $\{\bm{X}_k \}_{k=1}^m \subset \mathbb{R}^{N\times F} $, the $\mathbb{S}$-span of $\{\bm{X}_k \}_{k=1}^m$ is defined as $\mathrm{span}^{\mathbb{S}} \{\bm{X}_1, \dots, \bm{X}_m\} \vcentcolon = \{ \sum\limits_{k=1}^{m} \bm{X}_k C_k: C_k \in  \mathbb{S} \}$
\end{definition}
Given the above definition, the order-$m$ block Krylov subspace with respect to $A\in \mathbb{R}^{N\times N},\bm{B}\in \mathbb{R}^{N\times F}$, and  $\mathbb{S}$ can be defined as $\mathcal{K}_m^{\mathbb{S}} (A,\bm{B}) \vcentcolon = \text{span}^{\mathbb{S}} \{\bm{B}, A\bm{B}, \dots, A^{m-1} \bm{B}\} $. The corresponding block Krylov matrix is defined as $K_m (A,\bm{B})\vcentcolon = [ \bm{B}, A\bm{B}, \dots, A^{m-1} \bm{B}]$.

%With different definitions of block inner product, we have different  $\mathcal{K}_m^{\mathbb{S}} (A,\bm{B})$ as we will see in Section \ref{conv_in_krylov}.  The one we use in this paper is the classical block inner product and the block Krylov subspace is $\mathcal{K}_m^{Cl} (A,\bm{B}) \vcentcolon= \{ \sum\limits_{k=1}^{m} A^k \bm{B} C_k : C_k \in \mathbb{R}^{F\times F}\}$.
\subsection{Spectral Graph Convolution in Block Krylov Subspace Form}
\label{conv_in_krylov}
In this section, we will show that any graph convolution with well-defined analytic spectral filter defined on $\boldmath{L} \in \mathbb{R}^{N\times N}$ can be written as the product of a block Krylov matrix with a learnable parameter matrix in a specific form.

For any real analytic scalar function $g$, its power series expansion around center 0 is
$$g(x) = \sum\limits_{n=0}^\infty a_n x^n = \sum\limits_{n=0}^\infty \frac{g^{(n) } (0)}{n!} x^n, \; \abs{x} < R$$
where $R$ is the radius of convergence. We can define a filter by $g$.
Let $\rho(L)$ denote the spectrum radius of $L$ and suppose $\rho(L) <R$. The spectral filter $g(L)\in \mathbb{R}^{N \times N}$ can be defined as
$$g(L) \vcentcolon = \sum\limits_{n=0}^\infty a_n L^n =\sum\limits_{n=0}^\infty \frac{g^{(n) } (0)}{n!} L^n, \; \rho(L) < R$$
According to the definition of spectral graph convolution in \eqref{def}, graph signal $\bm{X}$ is filtered by $g(L)$ in the following way,
$$ g(L) \bm{X} = \sum\limits_{n=0}^\infty \frac{g^{(n) } (0)}{n!} L^n \bm{X} =\left[\bm{X}, L\bm{X}, L^2\bm{X}, \cdots  \right] \left[\frac{g^{(0) } (0)}{0!} {I_{F}} , \frac{g^{(1) } (0)}{1!}{I_{F}}  ,\frac{g^{(2) } (0)}{2!}{I_{F}}  , \cdots \right]^T = A'B'$$
%\colvec{\frac{g^{(0) } (0)}{0!} {I_{F}} \\ \frac{g^{(1) } (0)}{1!}{I_{F}}  \\ \frac{g^{(2) } (0)}{2!}{I_{F}}  \\ \vdots  }
where $A' \in \mathbf{R}^{N \times \infty}$ and $B' \in \mathbf{R}^{\infty \times F}$. It is easy to see that  $A'$ is a block Krylov matrix and Range($A'B'$) $\subseteq$ Range($A'$).  We know that there exists a smallest $m$ such that \cite{gutknecht2009block,frommer2017block}
\begin{equation} \label{krylov}
\text{span}^{\mathbb{S}} \{\bm{X}, L\bm{X}, L^2\bm{X}, \cdots  \} = \text{span}^{\mathbb{S}} \{\bm{X}, L\bm{X}, L^2\bm{X}, \dots,L^{m-1}\bm{X} \},
\end{equation}
\ie{} for any $k\geq m$, $L^k \bm{X} \in \mathcal{K}_m^{\mathbb{S}} (L,\bm{X})$. $m$ depends on $L$ and $\bm{X}$, so we will write it as $m(L,\bm{X})$ later, yet here we still use $m$ for simplicity. From \eqref{krylov}, the convolution can be written as
\begin{equation}\label{eq5}
g(L) \bm{X} = \sum\limits_{n=0}^\infty \frac{g^{(n) } (0)}{n!} L^n \bm{X} =\left[\bm{X}, L\bm{X}, L^2\bm{X}, \dots,L^{m-1}\bm{X}  \right] \left[ ({\Gamma_0}^{\mathbb{S}})^T  , ({\Gamma_1}^{\mathbb{S}})^T   , ({\Gamma_2}^{\mathbb{S}})^T ,\cdots ,({\Gamma_{m-1}^{\mathbb{S}}})^T  \right]^T \equiv  K_m (L,\bm{X}) \Gamma^{\mathbb{S}}
\end{equation}
%\colvec{ {\Gamma_0}^{\mathbb{S}}  \\ {\Gamma_1}^{\mathbb{S}}   \\ {\Gamma_2}^{\mathbb{S}} \\ \vdots \\ {\Gamma_{m-1}^{\mathbb{S}}} }
where ${\Gamma_{i}^{\mathbb{S}}},\; i=1,\dots, m-1$ are parameter matrix blocks and ${\Gamma_{i}^{\mathbb{S}}} \in \mathbb{R}^{F\times F}$ under classical definition of inner product. Then%If we use global inner product, $\Gamma_i^{\mathbb{S}} = c_i {I_F}$ and if we use loop-interchange definition, we will have a diagonal $\Gamma_i^{\mathbb{S}}$. When $g(L)\bm{X}$ is multiplied by a parameter matrix $W' \in \mathbb{R}^{F \times O}$,
\begin{equation} \label{eq6}
g(L)\bm{X}W' =  K_m (L,\bm{X}) \Gamma^{\mathbb{S}} W' =  K_m (L,\bm{X}) W^{\mathbb{S}}
\end{equation}
where $W^{\mathbb{S}} \equiv \Gamma^{\mathbb{S}} W' \in \mathbb{R}^{mF \times O}$. The essential  number of learnable parameters is $mF\times O$. %for $W^{\mathbb{S}}$ under classical inner product; 2) $FO+m$  for $c_i,W'$ under global inner product; and 3) $mF + FO$ for  $\text{diag}(\Gamma_i^{\mathbb{S}}),W'$ under loop-interchange inner product. We use the classical one for implementation.
\subsection{Deep GCN in Block Krylov Subspace Form}
\label{deep_gcn_krylov}
Since the spectral graph convolution can be simplified as \eqref{eq5}\eqref{eq6}, we can build deep GCN in the following way.
	
Suppose we have a sequence of analytic spectral filters $G=\{ g_0, g_1, \dots, g_n\}$ and a sequence of point-wise nonlinear activation functions $H = \{h_0, h_1,\dots, h_n \}$. Then,
\begin{equation} \label{eq7}
\textbf{Y} = \mbox{softmax} \left\lbrace g_n(L) \; h_{n-1} \left\lbrace  \cdots g_2(L) \; h_1\left\lbrace g_1(L) \; h_0 \left\{ g_0(L) \bm{X} W_0' \right\} W_1' \right\rbrace  W_2' \cdots \right\rbrace W_n' \right\rbrace
\end{equation}
Let us define $\bm{H_0}=\bm{X}$ and $\bm{H_{i+1}}=h_{i} \{ g_i(L) \bm{H_i} W_i'\}, i = 0,\dots,n-1$. Then $\bm{Y}=\mbox{softmax}\{ g_n(L)\bm{H_n} W_n' \}$. From \eqref{eq6}\eqref{eq7}, we have an iterative relation that $\bm{H_{i+1}} = h_{i} \{K_{m_i} (L,\bm{H_i}) W_i^{\mathbb{S}}\},$ where $m_i = m(L,\bm{H_i})$. It is easy to see that, when $g_i(L) = I$, \eqref{eq7} is fully connected network \cite{li2018deeper}; when $g_i(L) = \tilde{A}, \; n=1$, it is just GCN \cite{kipf2016classification}; when $g_i(L)$ is defined by Chebyshev polynomial \cite{hammond2011wavelets}, $W_i' = I$ and under the global inner product, \eqref{eq7} is ChebNet \cite{defferrard2016fast}.

\subsection{Difficulties in Computation}
\label{difficulty}
In the last subsection, we gave a general form of deep GCN in block Krylov form. Following this idea, we can leverage the existing block Arnoldi (Lanczos) algorithm \cite{frommer2017block, frommer2017radau} to compute orthogonal basis of $\mathcal{K}_{m_i}^{\mathbb{S}} (L,\bm{H_i})$ and find $m_i$. But there are some difficulties in practice:
\begin{enumerate}[leftmargin=12pt]
\item During the training phase, $\bm{H_i}$ changes every time that parameters are updated. This makes $m_i$ become a variable and thus requires adaptive size for parameter matrices.
\item %The value of $m_i$ can be large when we use global or loop-interchange block inner product. And
For classical inner product, the $QR$ factorization that is needed in block Arnoldi algorithm \cite{frommer2017block} is difficult to be put into backpropagation framework.
\end{enumerate}

Although direct implementation of block Krylov methods in GCN is hard, it inspires us that if we have a good way to stack multi-scale information in each hidden layer, the network will have the ability to be extended to deep architectures. We propose a way to alleviate the difficulties in section \ref{deep}.
%As Figure 2 shows, diffusion operator of the three popular citation networks are not low-rank; and also for Erd\H{o}s-R\'enyi graph $G(n,p)$ with $p = \omega(\frac{1}{n})$ and $\sigma^2 = p(1-p)$, the empirical spectral distribution of $\frac{1}{\sqrt{n\sigma}} A_n$ converges to Wigner semicircle distribution \cite{tran2013sparse}, which is far from low rank.
%Therefore, we propose a densely connected GCN and truncated block Krylov GCN which can efficiently stack multi-scale information with full $L$ and obtain a deep architecture.
\section{Deep GCN Architectures}
\label{deep}
Upon the analyses in the last section, we propose two architectures: \textit{snowball} and \textit{truncated Krylov}. These methods concatenate multi-scale feature information in hidden layers differently while both having the potential to be scaled to deeper architectures.
\subsection{Snowball}
\label{snowball_gcn}

In order to concatenate multi-scale features together and get a richer representation for each node, we design a densely connected graph network (Figure \ref{deep_gcn}(a)) as follows,
\begin{align}
\label{snowball}
&\bm{H_0} = \bm{X},\; \bm{H_{l+1}} = f \left( L \left[ \bm{H_0}, \bm{H_1},\dots, \bm{H_l} \right] W_l \right),\; l=0,1,2,\dots,n - 1 \nonumber \\
& \bm{C} = g \left( \left[ \bm{H_0}, \bm{H_1},\dots, \bm{H_n} \right] W_n \right)\\
&output = \text{softmax} \left( L^{p} \bm{C} W_C \right) \nonumber
\end{align}
where $W_l \in \mathbb{R}^{\left( \sum\limits_{i=0}^l F_i \right)  \times F_{l+1} }, W_n \in \mathbb{R}^{\left( \sum\limits_{i=0}^n F_i \right)  \times F_C}, W_C \in \mathbb{R}^{F_C \times F_O}$ are learnable parameter matrices, $F_{l+1}$ is the number of output channels in layer $l$; $f, g$ are point-wise activation functions; $C$ is a classifier of any kind; $p \in \{0,1\}$. $\bm{H_0}, \bm{H_1},\dots, \bm{H_n}$ are extracted features. $C$ can be a fully connected neural network or even an identity layer with $C=[\bm{H_0}, \bm{H_1},\dots, \bm{H_n}]$. When $p=0$, $L^p=I$ and when $p=1$, $L^P=L$, which means that we project $C$ back onto graph Fourier basis which is necessary when graph structure encodes much information. Following this construction, we can stack all learned features as the input of the subsequent hidden layer, which is an efficient way to concatenate multi-scale information. The size of input will grow like a snowball and this construction is similar to DenseNet \cite{huang2017densely}, which is designed for regular grids (images). Thus, some advantages of DenseNet are naturally inherited, \eg{} alleviate the vanishing-gradient problem, encourage feature reuse, increase the variation of input for each hidden layer, reduce the number of parameters, strengthen feature propagation and improve model compactness.

\subsection{Truncated Krylov}

As stated in Section \ref{difficulty}, the fact that $m_i$ is a variable makes GCN difficult to be merged into the block Krylov framework. But we can make a compromise and set $m_i$ as a hyperparameter. Then we can get a truncated block Krylov network (Figure \ref{deep_gcn}(b)) as shown below,
\begin{align}
&\bm{H_0} = \bm{X}, \; \bm{H_{l+1}}=f \left( \left[ \bm{H_l}, L \bm{H_l} \dots, L^{m_l -1} \bm{H_l} \right] W_l   \right), \; l=0,1,2,\dots,n - 1 \nonumber \\
& \bm{C} = g \left( \bm{H_n} W_n   \right)\\
&output = \text{softmax} \left(L^{p} \bm{C} W_C \right) \nonumber
\end{align}
where $W_l \in \mathbb{R}^{\left( m_l F_l \right)  \times F_{l+1} }, W_n \in \mathbb{R}^{ F_n \times F_C}, W_C \in \mathbb{R}^{F_C \times F_O}$ are learnable parameter matrices, $f$ and $g$ are activation functions, and $p \in \{0,1\}$.

There are many works on the analysis of error bounds of doing truncation in block Krylov methods \cite{frommer2017block}. But the results need many assumptions either on $\bm{X}$, \eg{} $\bm{X}$ is a standard Gaussian matrix \cite{ wang2015improved}, or on  $L$, \eg{} some conditions on the smallest and largest eigenvalues of $L$  have to be satisfied \cite{musco2018stability}. Instead of doing truncation for a specific function or a fixed $\bm{X}$, we are dealing with variable $\bm{X}$ during training. So we cannot put any restriction on $\bm{X}$ and its relation to $L$ to get a practical error bound.

\begin{figure*}[htbp]
\centering
\subfloat[Snowball]{
\captionsetup{justification = centering}
\includegraphics[width=0.49\textwidth]{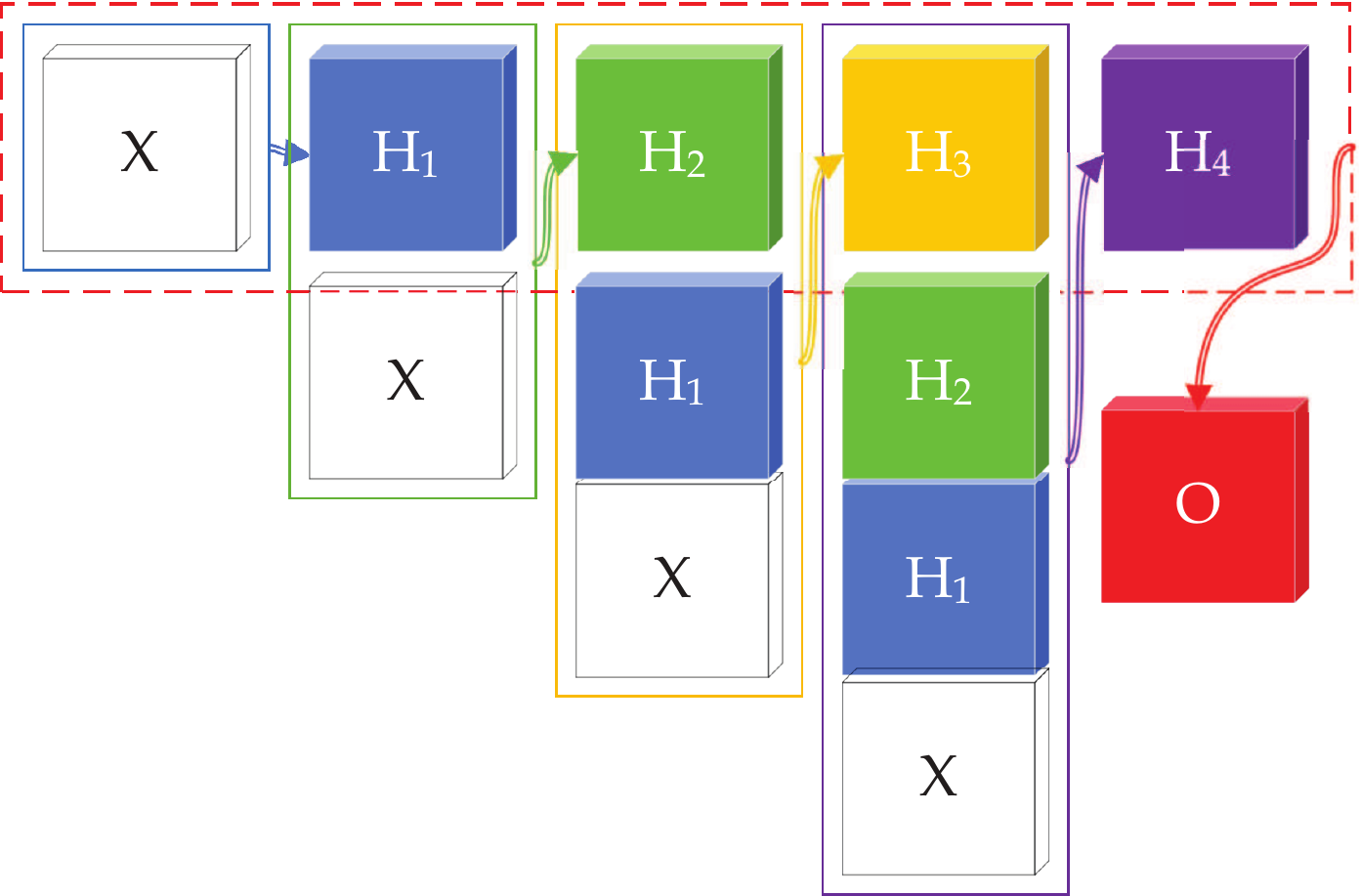}}
\hfill
\subfloat[Truncated Block Krylov]{
\captionsetup{justification = centering}
\includegraphics[width=0.49\textwidth]{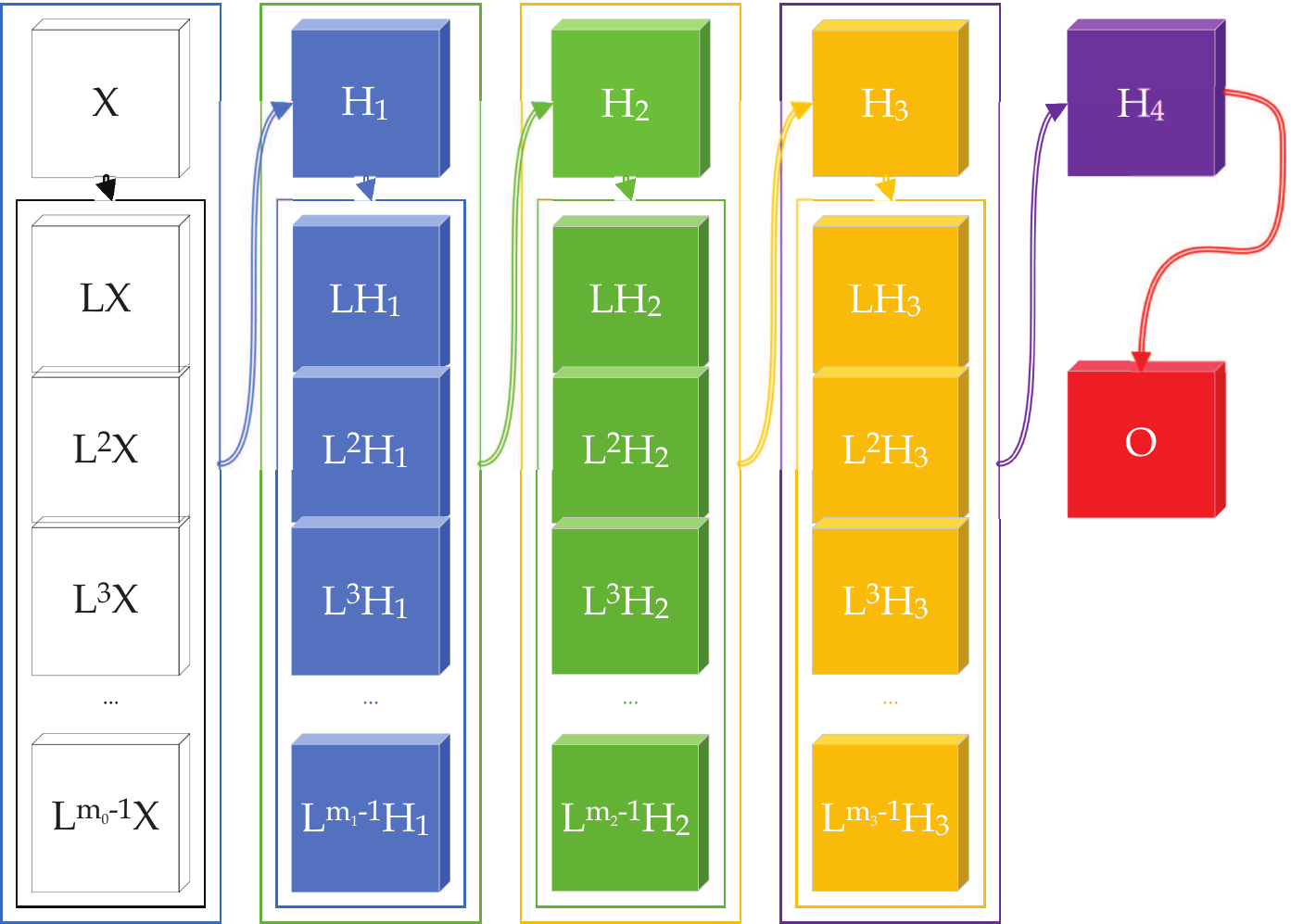}}
\caption{Deep GCN Architectures}
\label{deep_gcn}
\end{figure*}

Here we would like to mention \cite{liao2019lanczos}, which proposes to do low-rank approximation of $L$ by the Lanczos algorithm. The problem with this technique is that information in $L$ will be lost if $L$ is actually not low-rank. If we increase the Lanczos step to keep more information, it will hurt the efficiency of the algorithm. Since most of the graphs we are dealing with have sparse connectivity structures, they are actually not low-rank, \eg{} Erd\H{o}s-R\'enyi graph $G(n,p)$ with $p = \omega(\frac{1}{n})$ \cite{tran2013sparse} and Appendix \rom{4}.
Thus we did not do low-rank approximation in our computations.

\subsection{Equivalence of Linear Snowball GCN and Truncated Block Krylov Network}
\label{linear_snowball}
%Figure \ref{activation_functions} diffusion operator $L$ and ReLU are deadly duo for deep GCN. If we drop $L$, we will get a fully-connected networks which underperform GCN \cite{li2018deeper}. So, we can try to drop ReLU and just use identity function between layers. Under densely connected architecture, the extracted feature will become

Under the snowball GCN architecture, the identity function outperforms ReLU as shown in Figure \ref{activation_functions} and it is easier to train than Tanh. In this part, we will show that a multi-layer linear snowball GCN with identity function as $f$, identity layer as $C$ and $p=1$ is equivalent to a 1-layer block Krylov network with identity layer $C$, $p =1$ and a special parameter matrix.

We write $W_i$ as $W_i=\left[ (\bm{W_i^1})^T, \cdots , (\bm{W_i^{i+1}})^T \right]^T$ and follow \eqref{snowball} we have
%\colvec{\bm{W_i^1}\\ \vdots \\ \bm{W_i^{i+1}}}
\begin{equation*}
 \bm{H_0} = \bm{X}, \; \bm{H_1} = L\bm{X} W_0, \; \bm{H_2} = L[\bm{X}, \bm{H_1} ]W_1 = L\bm{X} \bm{W_1^1} +  L^2 \bm{X}\bm{ W_0 W_1^2} = L [\bm{X}, L\bm{X}]
 \begin{bmatrix}
 \bm{I} & 0\\
 0 & \bm{W_0}
\end{bmatrix} \colvec{ \bm{W_1^1} \\ \bm{W_1^2} }, \dots %H_3 = L[X, H_1, H_2] W_2
\end{equation*}
As in \eqref{snowball}, we have $L\bm{C}W_C = L[\bm{H_0}, \bm{H_1},\dots, \bm{H_n}] W_C$. Thus we can write
\iffalse
\[\bm{H_3} = L\bm{X} \bm{W_2^1} + L \bm{H_1} \bm{W_2^2} + L \bm{H_2} \bm{W_2^3} =
%LXW_2^1 + L^2 X (W_0 W_2^2 + W_1^2 W_2^3) + L^3 X W_0 W_1^2 W_2^3 =
[\bm{X}, L\bm{X}, L^2 \bm{X}]
\begin{bmatrix}
 \bm{I} & 0 & 0\\
 0 & \bm{I} & 0\\
 0 & 0 & \bm{W_0}
\end{bmatrix}
\begin{bmatrix}
 \bm{I} & 0 & 0\\
 0 & \bm{W_1^2} & 0\\
 0 & 0 & \bm{W_1^1}
\end{bmatrix} \colvec{ \bm{W_2^1} \\ \bm{W_2^2} \\ \bm{W_2^3}}
\]
\fi
\begin{align*}
 & [\bm{H_0}, \bm{H_1} \cdots,  \bm{H_n}] \\
  %\colvec{ W_{n}^1 \\  W_{n}^2\\ \vdots \\  W_{n}^{n} \\ W_{n}^{n+1} }
  =\ &  \ [\bm{X}, L\bm{X}, \cdots,  L^{n} \bm{X}]
%\begin{bmatrix}
%0 & 0 & \cdots & 0 & 0\\[6pt]
% 0 & \bm{I} & \cdots & 0 & 0\\[0pt]
% \vdots & \vdots & \ddots & \vdots & \vdots\\[6pt]
% 0 & 0 & \cdots & \bm{I} & 0\\[6pt]
% 0 & 0 & \cdots & 0 & \bm{I} \\
%\end{bmatrix}
\begin{bmatrix}
 \bm{I} & 0 & \cdots & 0 \\[6pt]
 0 & \bm{I} & \cdots & 0 \\[0pt]
 \vdots & \vdots & \ddots  & \vdots\\[6pt]
 0 & 0 & \cdots  & \bm{W_0} \\
\end{bmatrix}
\begin{bmatrix}
\bm{I} & 0 & \cdots & 0 \\[6pt]
0 & \bm{I} & \cdots & 0 \\[0pt]
 \vdots & \vdots & \ddots  & \vdots\\[6pt]
 0 & 0 & \cdots &  \bm{{W_1^1}} \\
\end{bmatrix}
\cdots
\begin{bmatrix}
 \bm{I} & 0 & \cdots & 0\\[6pt]
 0 & \bm{W_{n-1}^{n}} & \cdots & 0 \\[0pt]
 \vdots & \vdots & \ddots & \vdots\\[6pt]
 0 & 0 & \cdots &  \bm{W_{n-1}^1} \\
\end{bmatrix}
 %\colvec{ W_{n}^1 \\  W_{n}^2\\ \vdots \\  W_{n}^{n} \\ W_{n}^{n+1} }
\end{align*}
which is in the form of \eqref{eq6}, where the parameter matrix is the multiplication of a sequence of block diagonal matrices whose entries consist of identity blocks and blocks from other parameter matrices.

\subsection{Intuition behind Multi-scale Information Concatenation}
For each node $v$, denote $N(v)$ as the set of its neighbors. Then $L \bm{X}(v,:)$ can be interpreted as a weighted mean of $v$ and $N(v)$. If the networks goes deep as \eqref{eq1}, $Y'(v,:)$ becomes the weighted mean of $v$ and its $n-$hop neighbors (not exactly mean because we have ReLU in each layer). As the scope grows, the nodes in the same connected component tend to have the same (global) features, while losing their individual (local) features, which makes them indistinguishable. Although it is reasonable to assume that the nodes in the same cluster share many similar properties, it will be harmful if we ignore the "personal" differences between each node.

Therefore, to get a richer representation of each node, we propose to concatenate the multi-scale information together and the most naive architecture is the densely connected one. Truncated Krylov network works because in each layer $i$, we start the concatenation from $L^0 H_i$. By this way, the local information will not be diluted in each layer.
\section{Experiments}\label{sec:experiments}
We test linear snowball GCN ($f=g$=identity, $p=1$), snowball GCN ($f$=Tanh, $g$=identity, $p=1$) and truncated block Krylov network ($f=g$=Tanh, $p=0$) on public splits \cite{yang2016revisiting,liao2019lanczos} of Cora, Citeseer and PubMed\footnote{Source code to be found at \url{https://github.com/PwnerHarry/Stronger_GCN}}, as well as several smaller splits to increase the difficulty of the tasks \cite{liao2019lanczos,li2018deeper,sun2019stage}. We compare against several methods which allow validation, including graph convolutional networks for fingerprint (GCN-FP) \cite{duvenaud2015convolutional}, gated graph neural networks (GGNN) \cite{li2015gated}, diffusion convolutional neural networks (DCNN) \cite{atwood2015diffusion}, Chebyshev networks (Cheby) \cite{defferrard2016fast}, graph convolutional networks (GCN) \cite{kipf2016classification}, message passing neural networks (MPNN) \cite{gilmer2017neural}, graph sample and aggregate (GraphSAGE) \cite{hamilton2017inductive}, graph partition neural networks (GPNN) \cite{liao2018graph}, graph attention networks (GAT) \cite{velivckovic2017attention}, LanczosNet (LNet) \cite{liao2019lanczos} and AdaLanczosNet (AdaLNet) \cite{liao2019lanczos}. We also compare against some methods that do no use validation, including label propagation using ParWalks (LP) \cite{wu2012learning}, Co-training \cite{li2018deeper}, Self-training \cite{li2018deeper}, Union \cite{li2018deeper}, Intersection \cite{li2018deeper}, GCN without validation \cite{li2018deeper}, Multi-stage training \cite{sun2019stage}, Multi-stage self-supervised (M3S) training \cite{sun2019stage}, GCN with sparse virtual adversarial training (GCN-SVAT) \cite{sun2019virtual} and GCN with dense virtual adversarial training (GCN-DVAT) \cite{sun2019virtual}.

\begin{figure*}[htbp]
\centering
\subfloat[Linear Snowball]{
\captionsetup{justification = centering}
\includegraphics[width=0.32\textwidth]{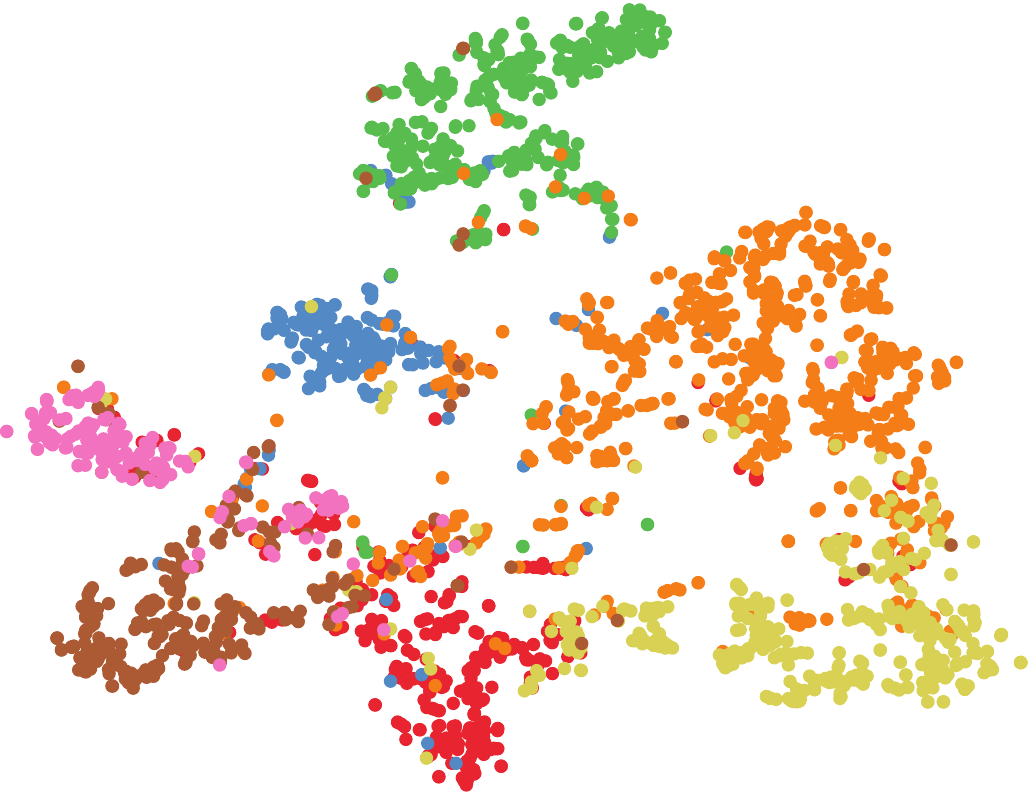}}
\hfill
\subfloat[Snowball]{
\captionsetup{justification = centering}
\includegraphics[width=0.32\textwidth]{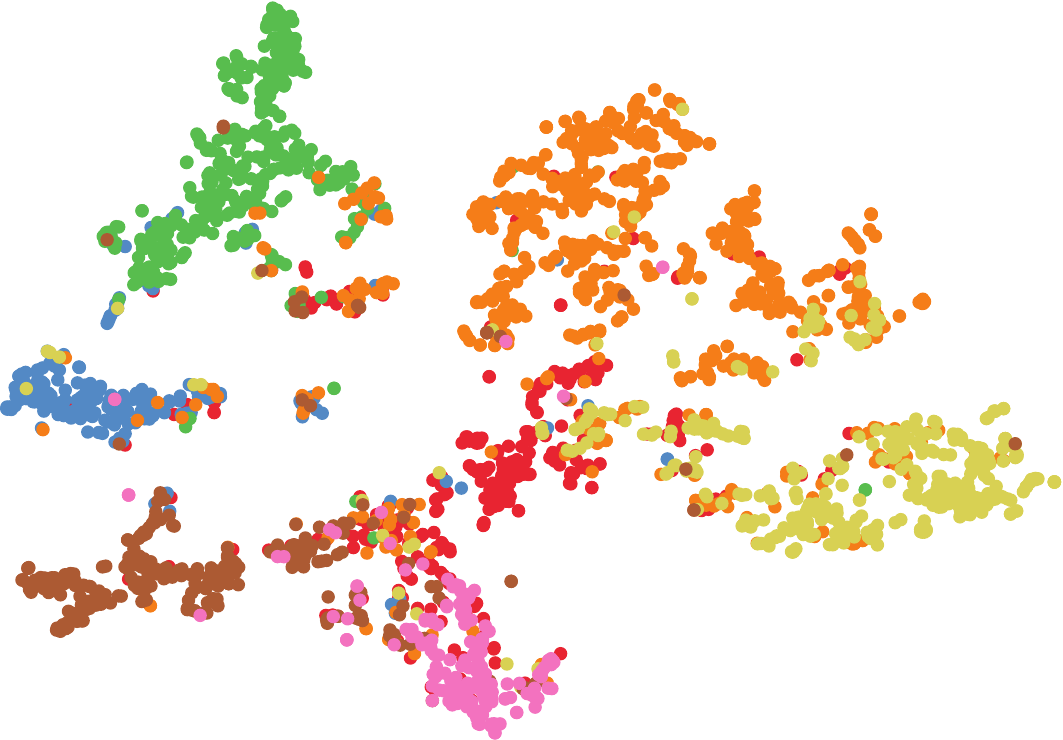}}
\hfill
\subfloat[Truncated Krylov]{
\captionsetup{justification = centering}
\includegraphics[width=0.32\textwidth]{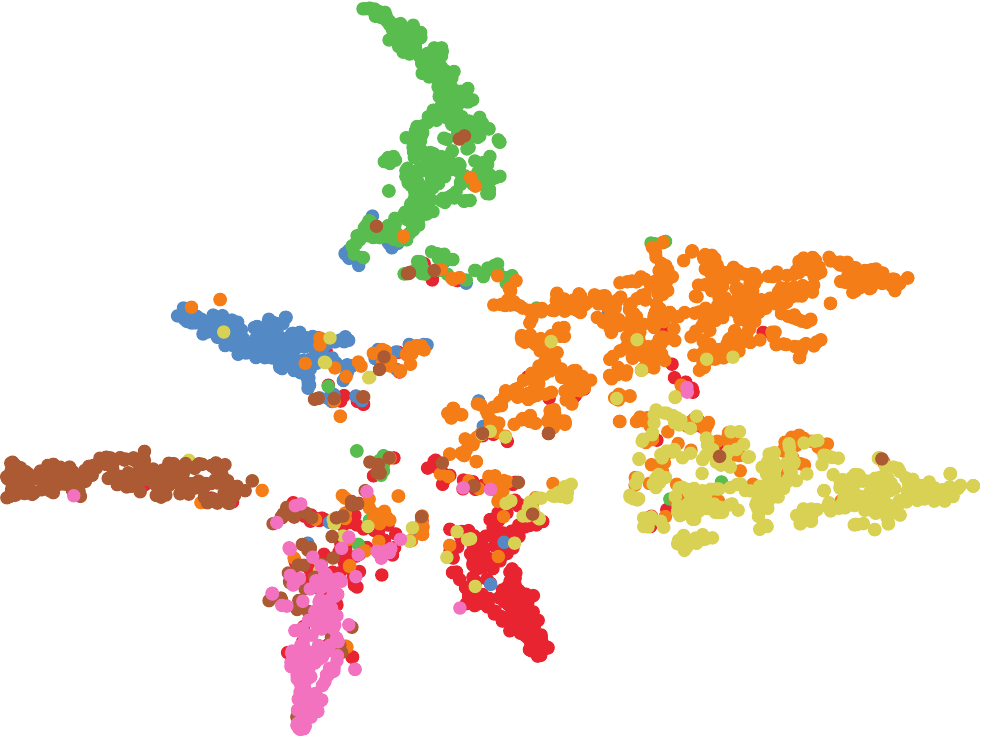}}
\caption{t-SNE for the extracted features trained on Cora (7 classes) public (5.2\%) using the best hyperparameter settings.}
\label{tsne}
\end{figure*}

\par
For each test case, we use the best hyperparameters to run $10$ independent times to get the average accuracy. The hyperparameters include learning rate and weight decay for the optimizer RMSprop or Adam, taking values in the intervals $[{10}^{-6}, 5\times{10}^{-3}]$ and $[{10}^{-5}, {10}^{-2}]$, respectively, width of hidden layers taking value in the set $\{50, \cdots, 5000\}$, number of hidden layers in the set $\{1, 2, \dots, 15\}$, dropout in $(0, 0.99]$, and the number of Krylov blocks taking value in $\{4, 5, \dots, 30\}$. The hyperparameter values of the test cases will be presented in the appendix.

To get achieve good training, we use adaptive number of episodes (but no more than $3000$): the early stopping counter is set to be 100.

We see that the proposed architectures achieve overwhelming performance in \textit{ALL} test cases. It is particularly worth noting that when the training sets are small, the proposed architectures perform astonishingly better than the existing methods. From the t-SNE\cite{maaten2008visualizing} visualization of the output layers in Figure \ref{tsne}, we see that the architectures extract good features with small training data, especially for truncated Krylov. What also impresses us is that linear Snowball GCN can achieve state-of-the-art performance with much less computational cost.

% Table generated by Excel2LaTeX from sheet 'Cora'
\begin{table*}[htbp]
\setlength{\tabcolsep}{1pt}
  \centering
  \caption{Accuracy Without Validation}
  \scriptsize
    \begin{tabular}{c|cccccc|cccccc|cccc}
    \multirow{2}[1]{*}{Algorithms} & \multicolumn{6}{c|}{Cora}                     & \multicolumn{6}{c|}{CiteSeer}                 & \multicolumn{4}{c}{PubMed} \\
          & 0.5\% & 1\%   & 2\%   & 3\%   & 4\%   & 5\%   & 0.5\% & 1\%   & 2\%   & 3\%   & 4\%   & 5\%   & 0.03\% & 0.05\% & 0.1\% & 0.3\% \\
    \midrule
    LP    & \cellcolor[rgb]{ .996,  .918,  .514}56.4 & \cellcolor[rgb]{ .992,  .796,  .494}62.3 & \cellcolor[rgb]{ .976,  .545,  .443}65.4 & \cellcolor[rgb]{ .973,  .412,  .42}67.5 & \cellcolor[rgb]{ .973,  .412,  .42}69.0 & \cellcolor[rgb]{ .973,  .412,  .42}70.2 & \cellcolor[rgb]{ .976,  .514,  .439}34.8 & \cellcolor[rgb]{ .973,  .412,  .42}40.2 & \cellcolor[rgb]{ .973,  .412,  .42}43.6 & \cellcolor[rgb]{ .973,  .412,  .42}45.3 & \cellcolor[rgb]{ .973,  .412,  .42}46.4 & \cellcolor[rgb]{ .973,  .412,  .42}47.3 & \cellcolor[rgb]{ .812,  .867,  .51}61.4 & \cellcolor[rgb]{ .82,  .871,  .51}66.4 & \cellcolor[rgb]{ .992,  .796,  .49}65.4 & \cellcolor[rgb]{ .973,  .412,  .42}66.8 \\
    Cheby & \cellcolor[rgb]{ .973,  .412,  .42}38.0 & \cellcolor[rgb]{ .973,  .412,  .42}52.0 & \cellcolor[rgb]{ .973,  .412,  .42}62.4 & \cellcolor[rgb]{ .98,  .596,  .455}70.8 & \cellcolor[rgb]{ .984,  .678,  .471}74.1 & \cellcolor[rgb]{ .992,  .792,  .49}77.6 & \cellcolor[rgb]{ .973,  .412,  .42}31.7 & \cellcolor[rgb]{ .976,  .482,  .431}42.8 & \cellcolor[rgb]{ .988,  .773,  .486}59.9 & \cellcolor[rgb]{ .996,  .863,  .506}66.2 & \cellcolor[rgb]{ .996,  .882,  .51}68.3 & \cellcolor[rgb]{ .996,  .886,  .51}69.3 & \cellcolor[rgb]{ .973,  .412,  .42}40.4 & \cellcolor[rgb]{ .973,  .412,  .42}47.3 & \cellcolor[rgb]{ .973,  .412,  .42}51.2 & \cellcolor[rgb]{ .984,  .698,  .475}72.8 \\
    Co-training & \cellcolor[rgb]{ 1,  .922,  .518}56.6 & \cellcolor[rgb]{ .949,  .91,  .518}66.4 & \cellcolor[rgb]{ .996,  .914,  .514}73.5 & \cellcolor[rgb]{ .996,  .882,  .51}75.9 & \cellcolor[rgb]{ .961,  .914,  .518}78.9 & \cellcolor[rgb]{ .871,  .886,  .514}80.8 & \cellcolor[rgb]{ .996,  .922,  .518}47.3 & \cellcolor[rgb]{ .992,  .847,  .502}55.7 & \cellcolor[rgb]{ .992,  .82,  .498}62.1 & \cellcolor[rgb]{ .992,  .78,  .49}62.5 & \cellcolor[rgb]{ .992,  .8,  .494}64.5 & \cellcolor[rgb]{ .992,  .804,  .494}65.5 & \cellcolor[rgb]{ .765,  .855,  .506}62.2 & \cellcolor[rgb]{ .671,  .827,  .502}68.3 & \cellcolor[rgb]{ .725,  .843,  .502}72.7 & \cellcolor[rgb]{ .824,  .871,  .51}78.2 \\
    Self-training & \cellcolor[rgb]{ .992,  .843,  .502}53.7 & \cellcolor[rgb]{ .969,  .914,  .518}66.1 & \cellcolor[rgb]{ .988,  .918,  .518}73.8 & \cellcolor[rgb]{ .929,  .902,  .514}77.2 & \cellcolor[rgb]{ .882,  .89,  .514}79.4 & \cellcolor[rgb]{ .996,  .914,  .514}80.0 & \cellcolor[rgb]{ .992,  .792,  .49}43.3 & \cellcolor[rgb]{ .996,  .914,  .514}58.1 & \cellcolor[rgb]{ .729,  .843,  .502}68.2 & \cellcolor[rgb]{ .812,  .871,  .51}69.8 & \cellcolor[rgb]{ .886,  .89,  .514}70.4 & \cellcolor[rgb]{ .953,  .91,  .518}71.0 & \cellcolor[rgb]{ .988,  .745,  .482}51.9 & \cellcolor[rgb]{ .988,  .753,  .482}58.7 & \cellcolor[rgb]{ .992,  .835,  .498}66.8 & \cellcolor[rgb]{ .996,  .898,  .51}77.0 \\
    Union & \cellcolor[rgb]{ .929,  .902,  .514}58.5 & \cellcolor[rgb]{ .733,  .847,  .506}69.9 & \cellcolor[rgb]{ .796,  .863,  .506}75.9 & \cellcolor[rgb]{ .765,  .855,  .506}78.5 & \cellcolor[rgb]{ .722,  .843,  .502}80.4 & \cellcolor[rgb]{ .702,  .835,  .502}81.7 & \cellcolor[rgb]{ .996,  .89,  .51}46.3 & \cellcolor[rgb]{ .937,  .906,  .518}59.1 & \cellcolor[rgb]{ .976,  .918,  .518}66.7 & \cellcolor[rgb]{ .996,  .871,  .506}66.7 & \cellcolor[rgb]{ .996,  .871,  .506}67.6 & \cellcolor[rgb]{ .996,  .863,  .506}68.2 & \cellcolor[rgb]{ .976,  .914,  .518}58.4 & \cellcolor[rgb]{ .996,  .914,  .514}64.0 & \cellcolor[rgb]{ .925,  .902,  .514}70.7 & \cellcolor[rgb]{ .588,  .804,  .494}79.2 \\
    Intersection & \cellcolor[rgb]{ .988,  .733,  .478}49.7 & \cellcolor[rgb]{ .996,  .898,  .51}65.0 & \cellcolor[rgb]{ .996,  .886,  .51}72.9 & \cellcolor[rgb]{ .941,  .906,  .518}77.1 & \cellcolor[rgb]{ .882,  .89,  .514}79.4 & \cellcolor[rgb]{ .984,  .918,  .518}80.2 & \cellcolor[rgb]{ .992,  .78,  .49}42.9 & \cellcolor[rgb]{ .937,  .906,  .518}59.1 & \cellcolor[rgb]{ .663,  .824,  .498}68.6 & \cellcolor[rgb]{ .749,  .851,  .506}70.1 & \cellcolor[rgb]{ .784,  .863,  .506}70.8 & \cellcolor[rgb]{ .89,  .89,  .514}71.2 & \cellcolor[rgb]{ .988,  .749,  .482}52.0 & \cellcolor[rgb]{ .988,  .773,  .486}59.3 & \cellcolor[rgb]{ .996,  .914,  .514}69.7 & \cellcolor[rgb]{ .965,  .914,  .518}77.6 \\
    MultiStage & \cellcolor[rgb]{ .831,  .875,  .51}61.1 & \cellcolor[rgb]{ .996,  .851,  .502}63.7 & \cellcolor[rgb]{ .933,  .902,  .514}74.4 & \cellcolor[rgb]{ .996,  .89,  .51}76.1 & \cellcolor[rgb]{ .992,  .843,  .502}77.2 &       & \cellcolor[rgb]{ .714,  .839,  .502}53.0 & \cellcolor[rgb]{ .996,  .906,  .514}57.8 & \cellcolor[rgb]{ .996,  .859,  .506}63.8 & \cellcolor[rgb]{ .996,  .902,  .514}68.0 & \cellcolor[rgb]{ .996,  .898,  .51}69.0 &       & \cellcolor[rgb]{ .996,  .906,  .514}57.4 & \cellcolor[rgb]{ .988,  .922,  .518}64.3 & \cellcolor[rgb]{ .976,  .918,  .518}70.2 &  \\
    M3S   & \cellcolor[rgb]{ .816,  .871,  .51}61.5 & \cellcolor[rgb]{ .902,  .894,  .514}67.2 & \cellcolor[rgb]{ .824,  .871,  .51}75.6 & \cellcolor[rgb]{ .855,  .882,  .51}77.8 & \cellcolor[rgb]{ .996,  .886,  .51}78.0 &       & \cellcolor[rgb]{ .561,  .796,  .494}56.1 & \cellcolor[rgb]{ .698,  .835,  .502}62.1 & \cellcolor[rgb]{ .996,  .918,  .514}66.4 & \cellcolor[rgb]{ .706,  .839,  .502}70.3 & \cellcolor[rgb]{ .863,  .882,  .51}70.5 &       & \cellcolor[rgb]{ .929,  .902,  .514}59.2 & \cellcolor[rgb]{ .98,  .918,  .518}64.4 & \cellcolor[rgb]{ .937,  .906,  .518}70.6 &  \\
    GCN   & \cellcolor[rgb]{ .976,  .537,  .443}42.6 & \cellcolor[rgb]{ .98,  .596,  .455}56.9 & \cellcolor[rgb]{ .984,  .655,  .463}67.8 & \cellcolor[rgb]{ .992,  .824,  .498}74.9 & \cellcolor[rgb]{ .996,  .863,  .506}77.6 & \cellcolor[rgb]{ .996,  .878,  .506}79.3 & \cellcolor[rgb]{ .973,  .467,  .427}33.4 & \cellcolor[rgb]{ .98,  .588,  .451}46.5 & \cellcolor[rgb]{ .992,  .831,  .498}62.6 & \cellcolor[rgb]{ .996,  .875,  .506}66.9 & \cellcolor[rgb]{ .996,  .894,  .51}68.7 & \cellcolor[rgb]{ .996,  .894,  .51}69.6 & \cellcolor[rgb]{ .98,  .584,  .451}46.4 & \cellcolor[rgb]{ .973,  .482,  .431}49.7 & \cellcolor[rgb]{ .976,  .549,  .443}56.3 & \cellcolor[rgb]{ .996,  .878,  .51}76.6 \\
    GCN-SVAT & \cellcolor[rgb]{ .98,  .565,  .447}43.6 & \cellcolor[rgb]{ .973,  .482,  .431}53.9 & \cellcolor[rgb]{ .992,  .82,  .498}71.4 & \cellcolor[rgb]{ .996,  .863,  .506}75.6 & \cellcolor[rgb]{ .996,  .902,  .514}78.3 & \cellcolor[rgb]{ .992,  .835,  .498}78.5 & \cellcolor[rgb]{ .996,  .914,  .514}47.0 & \cellcolor[rgb]{ .988,  .753,  .482}52.4 & \cellcolor[rgb]{ .996,  .902,  .514}65.8 & \cellcolor[rgb]{ .996,  .914,  .514}68.6 & \cellcolor[rgb]{ .996,  .91,  .514}69.5 & \cellcolor[rgb]{ .996,  .918,  .514}70.7 & \cellcolor[rgb]{ .988,  .749,  .482}52.1 & \cellcolor[rgb]{ .984,  .702,  .475}56.9 & \cellcolor[rgb]{ .988,  .745,  .482}63.5 & \cellcolor[rgb]{ .996,  .906,  .514}77.2 \\
    GCN-DVAT & \cellcolor[rgb]{ .988,  .714,  .475}49 & \cellcolor[rgb]{ .992,  .78,  .49}61.8 & \cellcolor[rgb]{ .992,  .839,  .502}71.9 & \cellcolor[rgb]{ .996,  .882,  .51}75.9 & \cellcolor[rgb]{ .996,  .906,  .514}78.4 & \cellcolor[rgb]{ .992,  .843,  .502}78.6 & \cellcolor[rgb]{ .788,  .863,  .506}51.5 & \cellcolor[rgb]{ .984,  .918,  .518}58.5 & \cellcolor[rgb]{ .859,  .882,  .51}67.4 & \cellcolor[rgb]{ .937,  .906,  .518}69.2 & \cellcolor[rgb]{ .784,  .863,  .506}70.8 & \cellcolor[rgb]{ .859,  .882,  .51}71.3 & \cellcolor[rgb]{ .992,  .784,  .49}53.3 & \cellcolor[rgb]{ .988,  .753,  .482}58.6 & \cellcolor[rgb]{ .992,  .82,  .498}66.3 & \cellcolor[rgb]{ .996,  .914,  .514}77.3 \\
    \midrule
    \textit{\textbf{linear Snowball}} & \cellcolor[rgb]{ .518,  .784,  .49}\textit{\textbf{69.5}} & \cellcolor[rgb]{ .475,  .773,  .49}\textit{\textbf{74.1}} & \cellcolor[rgb]{ .471,  .769,  .49}\textit{\textbf{79.4}} & \cellcolor[rgb]{ .529,  .788,  .494}\textit{\textbf{80.4}} & \cellcolor[rgb]{ .576,  .8,  .494}\textit{\textbf{81.3}} & \cellcolor[rgb]{ .608,  .812,  .498}\textit{\textbf{82.2}} & \cellcolor[rgb]{ .529,  .788,  .494}\textit{\textbf{56.8}} & \cellcolor[rgb]{ .431,  .761,  .486}\textit{\textbf{65.4}} & \cellcolor[rgb]{ .631,  .816,  .498}\textit{\textbf{68.8}} & \cellcolor[rgb]{ .561,  .796,  .494}\textit{\textbf{71.0}} & \cellcolor[rgb]{ .424,  .757,  .486}\textit{\textbf{72.2}} & \cellcolor[rgb]{ .576,  .8,  .494}\textit{\textbf{72.2}} & \cellcolor[rgb]{ .659,  .824,  .498}\textit{\textbf{64.1}} & \cellcolor[rgb]{ .573,  .8,  .494}\textit{\textbf{69.5}} & \cellcolor[rgb]{ .706,  .839,  .502}\textit{\textbf{72.9}} & \cellcolor[rgb]{ .557,  .796,  .494}\textit{\textbf{79.3}} \\
    \textit{\textbf{Snowball}} & \cellcolor[rgb]{ .608,  .808,  .498}\textit{\textbf{67.2}} & \cellcolor[rgb]{ .518,  .784,  .49}\textit{\textbf{73.5}} & \cellcolor[rgb]{ .553,  .792,  .494}\textit{\textbf{78.5}} & \cellcolor[rgb]{ .584,  .804,  .494}\textit{\textbf{80.0}} & \cellcolor[rgb]{ .549,  .792,  .494}\textit{\textbf{81.5}} & \cellcolor[rgb]{ .678,  .831,  .502}\textit{\textbf{81.8}} & \cellcolor[rgb]{ .549,  .792,  .494}\textit{56.4} & \cellcolor[rgb]{ .463,  .769,  .49}\textit{\textbf{65.0}} & \cellcolor[rgb]{ .514,  .784,  .49}\textit{\textbf{69.5}} & \cellcolor[rgb]{ .541,  .792,  .494}\textit{71.1} & \cellcolor[rgb]{ .4,  .749,  .486}\textit{\textbf{72.3}} & \cellcolor[rgb]{ .388,  .745,  .482}\textit{\textbf{72.8}} & \cellcolor[rgb]{ .725,  .843,  .502}\textit{\textbf{62.9}} & \cellcolor[rgb]{ .667,  .827,  .502}\textit{\textbf{68.3}} & \cellcolor[rgb]{ .667,  .827,  .502}\textit{\textbf{73.3}} & \cellcolor[rgb]{ .486,  .776,  .49}\textit{\textbf{79.6}} \\
    \textit{\textbf{truncated Krylov}} & \cellcolor[rgb]{ .388,  .745,  .482}\textit{\textbf{73.0}} & \cellcolor[rgb]{ .388,  .745,  .482}\textit{\textbf{75.5}} & \cellcolor[rgb]{ .388,  .745,  .482}\textit{\textbf{80.3}} & \cellcolor[rgb]{ .388,  .745,  .482}\textit{\textbf{81.5}} & \cellcolor[rgb]{ .388,  .745,  .482}\textit{\textbf{82.5}} & \cellcolor[rgb]{ .388,  .745,  .482}\textit{\textbf{83.4}} & \cellcolor[rgb]{ .388,  .745,  .482}\textit{\textbf{59.6}} & \cellcolor[rgb]{ .388,  .745,  .482}\textit{\textbf{66.0}} & \cellcolor[rgb]{ .388,  .745,  .482}\textit{\textbf{70.2}} & \cellcolor[rgb]{ .388,  .745,  .482}\textit{\textbf{71.8}} & \cellcolor[rgb]{ .388,  .745,  .482}\textit{\textbf{72.4}} & \cellcolor[rgb]{ .565,  .796,  .494}\textit{\textbf{72.2}} & \cellcolor[rgb]{ .388,  .745,  .482}\textit{\textbf{69.1}} & \cellcolor[rgb]{ .388,  .745,  .482}\textit{\textbf{71.8}} & \cellcolor[rgb]{ .388,  .745,  .482}\textit{\textbf{76.1}} & \cellcolor[rgb]{ .388,  .745,  .482}\textit{\textbf{80.0}} \\
    \bottomrule
    \bottomrule
    \multicolumn{17}{m{0.75\textwidth}}{\scriptsize For each (column), the greener the cell, the better the performance. The redder, the worse. If our methods achieve better performance than all others, the corresponding cell will be in bold.}
    \end{tabular}%    
\label{tab:results_no_validation}%
\end{table*}%

% Table generated by Excel2LaTeX from sheet 'Cora'
\begin{table*}[htbp]
\setlength{\tabcolsep}{1.5pt}
  \centering
  \caption{Accuracy With Validation}
  \scriptsize
    \begin{tabular}{c|cccc|ccc|cccc}
    \multirow{3}[1]{*}{Algorithms} & \multicolumn{4}{c|}{Cora}     & \multicolumn{3}{c|}{CiteSeer} & \multicolumn{4}{c}{PubMed} \\
          & \multirow{2}[1]{*}{0.5\%} & \multirow{2}[1]{*}{1\%} & \multirow{2}[1]{*}{3\%} & 5.2\% & \multirow{2}[1]{*}{0.5\%} & \multirow{2}[1]{*}{1\%} & 3.6\% & \multirow{2}[1]{*}{0.03\%} & \multirow{2}[1]{*}{0.05\%} & \multirow{2}[1]{*}{0.1\%} & 0.3\% \\
          &       &       &       & \textit{public} &       &       & \textit{public} &       &       &       & \textit{public} \\
    \midrule
    Cheby & \cellcolor[rgb]{ .973,  .412,  .42}33.9 & \cellcolor[rgb]{ .973,  .412,  .42}44.2 & \cellcolor[rgb]{ .976,  .545,  .443}62.1 & \cellcolor[rgb]{ .988,  .753,  .482}78.0 & \cellcolor[rgb]{ 1,  .922,  .518}45.3 & \cellcolor[rgb]{ 1,  .922,  .518}59.4 & \cellcolor[rgb]{ .835,  .875,  .51}70.1 & \cellcolor[rgb]{ .973,  .412,  .42}45.3 & \cellcolor[rgb]{ .973,  .412,  .42}48.2 & \cellcolor[rgb]{ .973,  .412,  .42}55.2 & \cellcolor[rgb]{ .973,  .412,  .42}69.8 \\
    GCN-FP & \cellcolor[rgb]{ .996,  .906,  .514}50.5 & \cellcolor[rgb]{ .992,  .843,  .502}59.6 & \cellcolor[rgb]{ .992,  .796,  .49}71.7 & \cellcolor[rgb]{ .973,  .42,  .42}74.6 & \cellcolor[rgb]{ .996,  .859,  .502}43.9 & \cellcolor[rgb]{ .988,  .718,  .478}54.3 & \cellcolor[rgb]{ .973,  .412,  .42}61.5 & \cellcolor[rgb]{ .996,  .851,  .502}56.2 & \cellcolor[rgb]{ .996,  .875,  .506}63.2 & \cellcolor[rgb]{ .996,  .847,  .502}70.3 & \cellcolor[rgb]{ .992,  .804,  .494}76.0 \\
    GGNN  & \cellcolor[rgb]{ .992,  .839,  .502}48.2 & \cellcolor[rgb]{ .996,  .871,  .506}60.5 & \cellcolor[rgb]{ .992,  .831,  .498}73.1 & \cellcolor[rgb]{ .988,  .714,  .475}77.6 & \cellcolor[rgb]{ .996,  .875,  .506}44.3 & \cellcolor[rgb]{ .992,  .784,  .49}56.0 & \cellcolor[rgb]{ .984,  .627,  .459}64.6 & \cellcolor[rgb]{ .992,  .835,  .498}55.8 & \cellcolor[rgb]{ .996,  .878,  .51}63.3 & \cellcolor[rgb]{ .996,  .851,  .502}70.4 & \cellcolor[rgb]{ .992,  .792,  .49}75.8 \\
    DCNN  & \cellcolor[rgb]{ .788,  .863,  .506}59.0 & \cellcolor[rgb]{ .835,  .875,  .51}66.4 & \cellcolor[rgb]{ .98,  .918,  .518}76.7 & \cellcolor[rgb]{ 1,  .922,  .518}79.7 & \cellcolor[rgb]{ .741,  .847,  .506}53.1 & \cellcolor[rgb]{ .812,  .867,  .51}62.2 & \cellcolor[rgb]{ .918,  .898,  .514}69.4 & \cellcolor[rgb]{ .863,  .882,  .51}60.9 & \cellcolor[rgb]{ .855,  .882,  .51}66.7 & \cellcolor[rgb]{ .953,  .91,  .518}73.1 & \cellcolor[rgb]{ .996,  .855,  .502}76.8 \\
    MPNN  & \cellcolor[rgb]{ .992,  .788,  .49}46.5 & \cellcolor[rgb]{ .988,  .761,  .486}56.7 & \cellcolor[rgb]{ .992,  .804,  .494}72.0 & \cellcolor[rgb]{ .988,  .753,  .482}78.0 & \cellcolor[rgb]{ .988,  .765,  .486}41.8 & \cellcolor[rgb]{ .988,  .718,  .478}54.3 & \cellcolor[rgb]{ .98,  .588,  .451}64.0 & \cellcolor[rgb]{ .988,  .757,  .486}53.9 & \cellcolor[rgb]{ .988,  .765,  .486}59.6 & \cellcolor[rgb]{ .988,  .761,  .486}67.3 & \cellcolor[rgb]{ .992,  .78,  .49}75.6 \\
    GraphSAGE & \cellcolor[rgb]{ .976,  .518,  .439}37.5 & \cellcolor[rgb]{ .976,  .545,  .443}49.0 & \cellcolor[rgb]{ .98,  .6,  .455}64.2 & \cellcolor[rgb]{ .973,  .412,  .42}74.5 & \cellcolor[rgb]{ .973,  .412,  .42}33.8 & \cellcolor[rgb]{ .98,  .588,  .451}51.0 & \cellcolor[rgb]{ .992,  .812,  .494}67.2 & \cellcolor[rgb]{ .973,  .416,  .42}45.4 & \cellcolor[rgb]{ .98,  .561,  .447}53.0 & \cellcolor[rgb]{ .988,  .706,  .475}65.4 & \cellcolor[rgb]{ .996,  .855,  .502}76.8 \\
    GAT   & \cellcolor[rgb]{ .984,  .635,  .463}41.4 & \cellcolor[rgb]{ .976,  .533,  .443}48.6 & \cellcolor[rgb]{ .973,  .412,  .42}56.8 & \cellcolor[rgb]{ .482,  .773,  .49}83.0 & \cellcolor[rgb]{ .98,  .604,  .455}38.2 & \cellcolor[rgb]{ .973,  .412,  .42}46.5 & \cellcolor[rgb]{ .553,  .796,  .494}72.5 & \cellcolor[rgb]{ .984,  .635,  .463}50.9 & \cellcolor[rgb]{ .973,  .478,  .431}50.4 & \cellcolor[rgb]{ .976,  .537,  .443}59.6 & \cellcolor[rgb]{ .647,  .824,  .498}79.0 \\
    GCN   & \cellcolor[rgb]{ 1,  .922,  .518}50.9 & \cellcolor[rgb]{ 1,  .922,  .518}62.3 & \cellcolor[rgb]{ 1,  .922,  .518}76.5 & \cellcolor[rgb]{ .878,  .886,  .514}80.5 & \cellcolor[rgb]{ .992,  .843,  .502}43.6 & \cellcolor[rgb]{ .988,  .757,  .486}55.3 & \cellcolor[rgb]{ 1,  .922,  .518}68.7 & \cellcolor[rgb]{ 1,  .922,  .518}57.9 & \cellcolor[rgb]{ 1,  .922,  .518}64.6 & \cellcolor[rgb]{ .969,  .914,  .518}73.0 & \cellcolor[rgb]{ 1,  .922,  .518}77.8 \\
    LNet  & \cellcolor[rgb]{ .812,  .867,  .51}58.1 & \cellcolor[rgb]{ .847,  .878,  .51}66.1 & \cellcolor[rgb]{ .918,  .898,  .514}77.3 & \cellcolor[rgb]{ .996,  .898,  .514}79.5 & \cellcolor[rgb]{ .737,  .847,  .506}53.2 & \cellcolor[rgb]{ .871,  .886,  .514}61.3 & \cellcolor[rgb]{ .988,  .741,  .482}66.2 & \cellcolor[rgb]{ .886,  .89,  .514}60.4 & \cellcolor[rgb]{ .706,  .839,  .502}68.8 & \cellcolor[rgb]{ .902,  .894,  .514}73.4 & \cellcolor[rgb]{ .855,  .882,  .51}78.3 \\
    AdaLNet & \cellcolor[rgb]{ .737,  .847,  .506}60.8 & \cellcolor[rgb]{ .792,  .863,  .506}67.5 & \cellcolor[rgb]{ .875,  .886,  .514}77.7 & \cellcolor[rgb]{ .89,  .89,  .514}80.4 & \cellcolor[rgb]{ .718,  .843,  .502}53.8 & \cellcolor[rgb]{ .737,  .847,  .506}63.3 & \cellcolor[rgb]{ 1,  .922,  .518}68.7 & \cellcolor[rgb]{ .859,  .882,  .51}61.0 & \cellcolor[rgb]{ .902,  .894,  .514}66.0 & \cellcolor[rgb]{ 1,  .922,  .518}72.8 & \cellcolor[rgb]{ .914,  .898,  .514}78.1 \\
    \midrule
    \textit{\textbf{linear Snowball}} & \cellcolor[rgb]{ .494,  .776,  .49}\textit{\textbf{70.0}} & \cellcolor[rgb]{ .565,  .796,  .494}\textit{\textbf{73.1}} & \cellcolor[rgb]{ .525,  .784,  .49}\textit{\textbf{81.0}} & \cellcolor[rgb]{ .455,  .765,  .486}\textit{\textbf{83.2}} & \cellcolor[rgb]{ .533,  .788,  .494}\textit{\textbf{59.4}} & \cellcolor[rgb]{ .561,  .796,  .494}\textit{\textbf{65.9}} & \cellcolor[rgb]{ .431,  .761,  .486}\textit{\textbf{73.5}} & \cellcolor[rgb]{ .529,  .788,  .494}\textit{\textbf{68.1}} & \cellcolor[rgb]{ .616,  .812,  .498}\textit{\textbf{70.0}} & \cellcolor[rgb]{ .831,  .875,  .51}\textit{\textbf{73.8}} & \cellcolor[rgb]{ .58,  .804,  .494}\textit{\textbf{79.2}} \\
    \textit{\textbf{Snowball}} & \cellcolor[rgb]{ .416,  .753,  .486}\textit{\textbf{73.0}} & \cellcolor[rgb]{ .416,  .753,  .486}\textit{\textbf{76.8}} & \cellcolor[rgb]{ .553,  .792,  .494}\textit{\textbf{80.7}} & \cellcolor[rgb]{ .388,  .745,  .482}\textit{\textbf{83.6}} & \cellcolor[rgb]{ .443,  .761,  .486}\textit{\textbf{62.1}} & \cellcolor[rgb]{ .671,  .827,  .502}\textit{\textbf{64.2}} & \cellcolor[rgb]{ .541,  .792,  .494}\textit{\textbf{72.6}} & \cellcolor[rgb]{ .404,  .753,  .486}\textit{\textbf{70.8}} & \cellcolor[rgb]{ .388,  .745,  .482}\textit{\textbf{73.2}} & \cellcolor[rgb]{ .388,  .745,  .482}\textit{\textbf{76.5}} & \cellcolor[rgb]{ .49,  .776,  .49}\textit{\textbf{79.5}} \\
    \textit{\textbf{truncated Krylov}} & \cellcolor[rgb]{ .388,  .745,  .482}\textit{\textbf{73.9}} & \cellcolor[rgb]{ .388,  .745,  .482}\textit{\textbf{77.4}} & \cellcolor[rgb]{ .388,  .745,  .482}\textit{\textbf{82.2}} & \cellcolor[rgb]{ .404,  .753,  .486}\textit{\textbf{83.5}} & \cellcolor[rgb]{ .388,  .745,  .482}\textit{\textbf{63.7}} & \cellcolor[rgb]{ .388,  .745,  .482}\textit{\textbf{68.4}} & \cellcolor[rgb]{ .388,  .745,  .482}\textit{\textbf{73.9}} & \cellcolor[rgb]{ .388,  .745,  .482}\textit{\textbf{71.1}} & \cellcolor[rgb]{ .416,  .753,  .486}\textit{\textbf{72.9}} & \cellcolor[rgb]{ .529,  .788,  .494}\textit{\textbf{75.7}} & \cellcolor[rgb]{ .388,  .745,  .482}\textit{\textbf{79.9}} \\
    \bottomrule
    \bottomrule
    \end{tabular}%
\label{tab:results_with_validation}%
\end{table*}%

%\section{Future Works}
%There are still some future works we need to do: 1) We need to further investigate how the point-wise nonlinear activation functions influence block vectors (\eg{} the feature block vector $\bm{X}$ and hidden feature block vectors $H_i$) so that we can find better activation functions than Tanh; 2) We will try to find a better way to leverage the block Krylov algorithms instead of simply doing truncation.

\clearpage
\bibliographystyle{abbrv}
\bibliography{references}

\begin{thebibliography}{10}

\bibitem{atwood2015diffusion}
J.~Atwood and D.~Towsley.
\newblock Diffusion-convolutional neural networks.
\newblock {\em arXiv}, abs/1511.02136, 2015.

\bibitem{bronstein2016geometric}
M.~M. Bronstein, J.~Bruna, Y.~LeCun, A.~Szlam, and P.~Vandergheynst.
\newblock Geometric deep learning: going beyond euclidean data.
\newblock {\em arXiv}, abs/1611.08097, 2016.

\bibitem{chen2018fastgcn}
J.~Chen, T.~Ma, and C.~Xiao.
\newblock Fastgcn: fast learning with graph convolutional networks via
  importance sampling.
\newblock {\em arXiv preprint arXiv:1801.10247}, 2018.

\bibitem{chen2017stochastic}
J.~Chen, J.~Zhu, and L.~Song.
\newblock Stochastic training of graph convolutional networks with variance
  reduction.
\newblock {\em arXiv preprint arXiv:1710.10568}, 2017.

\bibitem{coifman2006diffusionmaps}
R.~R. Coifman and S.~Lafon.
\newblock Diffusion maps.
\newblock {\em Applied and computational harmonic analysis}, 21(1):5--30, 2006.

\bibitem{coifman2006diffusion}
R.~R. Coifman and M.~Maggioni.
\newblock Diffusion wavelets.
\newblock {\em Applied and Computational Harmonic Analysis}, 21(1):53--94,
  2006.

\bibitem{defferrard2016fast}
M.~Defferrard, X.~Bresson, and P.~Vandergheynst.
\newblock Convolutional neural networks on graphs with fast localized spectral
  filtering.
\newblock {\em arXiv}, abs/1606.09375, 2016.

\bibitem{duvenaud2015convolutional}
D.~K. Duvenaud, D.~Maclaurin, J.~Iparraguirre, R.~Bombarell, T.~Hirzel,
  A.~Aspuru-Guzik, and R.~P. Adams.
\newblock Convolutional networks on graphs for learning molecular fingerprints.
\newblock In {\em Advances in neural information processing systems}, pages
  2224--2232, 2015.

\bibitem{frommer2017radau}
A.~Frommer, K.~Lund, M.~Schweitzer, and D.~B. Szyld.
\newblock The radau--lanczos method for matrix functions.
\newblock {\em SIAM Journal on Matrix Analysis and Applications},
  38(3):710--732, 2017.

\bibitem{frommer2017block}
A.~Frommer, K.~Lund, and D.~B. Szyld.
\newblock Block {Krylov} subspace methods for functions of matrices.
\newblock {\em Electronic Transactions on Numerical Analysis}, 47:100--126,
  2017.

\bibitem{gilmer2017neural}
J.~Gilmer, S.~S. Schoenholz, P.~F. Riley, O.~Vinyals, and G.~E. Dahl.
\newblock Neural message passing for quantum chemistry.
\newblock In {\em Proceedings of the 34th International Conference on Machine
  Learning-Volume 70}, pages 1263--1272. JMLR. org, 2017.

\bibitem{gutknecht2009block}
M.~H. Gutknecht and T.~Schmelzer.
\newblock The block grade of a block krylov space.
\newblock {\em Linear Algebra and its Applications}, 430(1):174--185, 2009.

\bibitem{hamilton2017inductive}
W.~L. Hamilton, R.~Ying, and J.~Leskovec.
\newblock Inductive representation learning on large graphs.
\newblock {\em arXiv}, abs/1706.02216, 2017.

\bibitem{hammond2011wavelets}
D.~K. Hammond, P.~Vandergheynst, and R.~Gribonval.
\newblock Wavelets on graphs via spectral graph theory.
\newblock {\em Applied and Computational Harmonic Analysis}, 30(2):129--150,
  2011.

\bibitem{hinton2006fast}
G.~E. Hinton, S.~Osindero, and Y.-W. Teh.
\newblock A fast learning algorithm for deep belief nets.
\newblock {\em Neural computation}, 18(7):1527--1554, 2006.

\bibitem{huang2017densely}
G.~Huang, Z.~Liu, L.~Van Der~Maaten, and K.~Q. Weinberger.
\newblock Densely connected convolutional networks.
\newblock In {\em Proceedings of the IEEE conference on computer vision and
  pattern recognition}, pages 4700--4708, 2017.

\bibitem{kipf2016classification}
T.~N. Kipf and M.~Welling.
\newblock Semi-supervised classification with graph convolutional networks.
\newblock {\em arXiv}, abs/1609.02907, 2016.

\bibitem{lecun2015deep}
Y.~LeCun, Y.~Bengio, and G.~Hinton.
\newblock Deep learning.
\newblock {\em nature}, 521(7553):436, 2015.

\bibitem{lecun1998gradient}
Y.~LeCun, L.~Bottou, Y.~Bengio, P.~Haffner, et~al.
\newblock Gradient-based learning applied to document recognition.
\newblock {\em Proceedings of the IEEE}, 86(11):2278--2324, 1998.

\bibitem{li2018deeper}
Q.~Li, Z.~Han, and X.~Wu.
\newblock Deeper insights into graph convolutional networks for semi-supervised
  learning.
\newblock {\em arXiv}, abs/1801.07606, 2018.

\bibitem{li2018adaptive}
R.~Li, S.~Wang, F.~Zhu, and J.~Huang.
\newblock Adaptive graph convolutional neural networks.
\newblock In {\em Thirty-Second AAAI Conference on Artificial Intelligence},
  2018.

\bibitem{li2015gated}
Y.~Li, D.~Tarlow, M.~Brockschmidt, and R.~Zemel.
\newblock Gated graph sequence neural networks.
\newblock {\em arXiv preprint arXiv:1511.05493}, 2015.

\bibitem{liao2018graph}
R.~Liao, M.~Brockschmidt, D.~Tarlow, A.~L. Gaunt, R.~Urtasun, and R.~Zemel.
\newblock Graph partition neural networks for semi-supervised classification.
\newblock {\em arXiv preprint arXiv:1803.06272}, 2018.

\bibitem{liao2019lanczos}
R.~Liao, Z.~Zhao, R.~Urtasun, and R.~S. Zemel.
\newblock Lanczosnet: Multi-scale deep graph convolutional networks.
\newblock {\em arXiv}, abs/1901.01484, 2019.

\bibitem{maaten2008visualizing}
L.~v.~d. Maaten and G.~Hinton.
\newblock Visualizing data using t-sne.
\newblock {\em Journal of machine learning research}, 9(Nov):2579--2605, 2008.

\bibitem{monti2017geometric}
F.~Monti, D.~Boscaini, J.~Masci, E.~Rodola, J.~Svoboda, and M.~M. Bronstein.
\newblock Geometric deep learning on graphs and manifolds using mixture model
  cnns.
\newblock In {\em Proceedings of the IEEE Conference on Computer Vision and
  Pattern Recognition}, pages 5115--5124, 2017.

\bibitem{musco2018stability}
C.~Musco, C.~Musco, and A.~Sidford.
\newblock Stability of the lanczos method for matrix function approximation.
\newblock In {\em Proceedings of the Twenty-Ninth Annual ACM-SIAM Symposium on
  Discrete Algorithms}, pages 1605--1624. Society for Industrial and Applied
  Mathematics, 2018.

\bibitem{shuman2012emerging}
D.~I. Shuman, S.~K. Narang, P.~Frossard, A.~Ortega, and P.~Vandergheynst.
\newblock The emerging field of signal processing on graphs: Extending
  high-dimensional data analysis to networks and other irregular domains.
\newblock {\em arXiv preprint arXiv:1211.0053}, 2012.

\bibitem{sun2019virtual}
K.~Sun, H.~Guo, Z.~Zhu, and Z.~Lin.
\newblock Virtual adversarial training on graph convolutional networks in node
  classification.
\newblock {\em arXiv preprint arXiv:1902.11045}, 2019.

\bibitem{sun2019stage}
K.~Sun, Z.~Zhu, and Z.~Lin.
\newblock Multi-stage self-supervised learning for graph convolutional
  networks.
\newblock {\em arXiv}, abs/1902.11038, 2019.

\bibitem{tran2013sparse}
L.~V. Tran, V.~H. Vu, and K.~Wang.
\newblock Sparse random graphs: Eigenvalues and eigenvectors.
\newblock {\em Random Structures \& Algorithms}, 42(1):110--134, 2013.

\bibitem{velivckovic2017attention}
P.~Veli{\v{c}}kovi{\'c}, G.~Cucurull, A.~Casanova, A.~Romero, P.~Lio, and
  Y.~Bengio.
\newblock Graph attention networks.
\newblock {\em arXiv}, abs/1710.10903, 2017.

\bibitem{wang2015improved}
S.~Wang, Z.~Zhang, and T.~Zhang.
\newblock Improved analyses of the randomized power method and block lanczos
  method.
\newblock {\em arXiv preprint arXiv:1508.06429}, 2015.

\bibitem{wu2012learning}
X.-M. Wu, Z.~Li, A.~M. So, J.~Wright, and S.-F. Chang.
\newblock Learning with partially absorbing random walks.
\newblock In {\em Advances in Neural Information Processing Systems}, pages
  3077--3085, 2012.

\bibitem{wu2019survey}
Z.~Wu, S.~Pan, F.~Chen, G.~Long, C.~Zhang, and P.~S. Yu.
\newblock A comprehensive survey on graph neural networks.
\newblock {\em arXiv}, abs/1901.00596, 2019.

\bibitem{yang2016revisiting}
Z.~Yang, W.~W. Cohen, and R.~Salakhutdinov.
\newblock Revisiting semi-supervised learning with graph embeddings.
\newblock {\em arXiv preprint arXiv:1603.08861}, 2016.

\bibitem{zhang2018graph}
S.~Zhang, H.~Tong, J.~Xu, and R.~Maciejewski.
\newblock Graph convolutional networks: Algorithms, applications and open
  challenges.
\newblock In {\em International Conference on Computational Social Networks},
  pages 79--91. Springer, 2018.

\end{thebibliography}

\newpage

\section*{Appendices}
\label{appendix}
\subsection*{Appendix \rom{1}: Proof of Theorem 1, 2} \label{appendix:1}
We extend Theorem 1 in \cite{li2018deeper} to a general diffusion operator $L$ in the following lemma.
\begin{lemma} 1
Suppose that a graph $\mathcal{G}$ has $k$ connected components $\{C_i\}_{i=1}^k$ and the diffusion operator $L$ is defined as that in \eqref{eq0}. $L$ has $k$ linearly independent eigenvectors $\{\bm{v_1},\dots, \bm{v_k}\}$ corresponding to its largest eigenvalue $\lambda_{max}$. If $\mathcal{G}$ has no bipartite components, then for any $\bm{x} \in \mathbb{R}^N$
\begin{equation}\label{eq2}
\underset{m \rightarrow \infty }{\text{lim}}\; (\frac{1}{\lambda_{max}} L )^m \bm{x} = [\bm{v_1},\dots, \bm{v_k}] \bm{\theta},
\end{equation}
for some $\bm{\theta} \in \mathbb{R}^k$.
\end{lemma}

%% {(\bf What does "uniformly sample" mean? A uniform distribution is for a finite region. Note that you cannot pick a real number uniformly in the 1-dimensional case), I agree, but please check: \url{https://math.stackexchange.com/questions/85955/is-there-a-uniform-distribution-over-the-real-line} and \url{https://en.wikipedia.org/wiki/Prior_probability#Improper_priors}}, then
%\begin{equation}\label{eq:two_vec}
%\doubleP\left(\text{rank}\left(\text{ReLU}([\bm{x},\bm{y}])\right) \leq %\text{rank}([\bm{x},\bm{y}]) \;|\;\bm{x},\bm{y} \in \mathbb{R}^N\right) = 1
%\end{equation}
%\end{lemma}
%\begin{proof}
%See Appendix \rom{1}.
%\end{proof}

\begin{lemma} 2
Suppose we  randomly sample $\bm{x}, \bm{y}\in \mathbb{R}^N$ under a continuous distribution. Suppose we have point-wise function $\text{ReLU}(z) = \text{max}(0,z)$, we have
$$\doubleP(\text{rank}\left(\text{ReLU}([\bm{x}, \bm{y}])\right) \leq \text{rank}([\bm{x}, \bm{y}]) \;|\; \bm{x}, \bm{y} \in \mathbb{R}^N) = 1$$
\end{lemma}
\begin{proof}
We generalize ReLU onto multi-dimensional case by applying it element-wise on every element of the matrix. Any $\bm{x} \in \mathbb{R}^N$ can be represented as $\bm{x} = \bm{x_+} + \bm{x_-}$, where $\bm{x_+}$ and $\bm{x_-}$ are the nonnegative and nonpositive components of $\bm{x}$, respectively. We can see that $\text{ReLU}(\bm{x}) = \bm{x}_+$. It is trivial when $\bm{y}=0$. We only discuss for all nonzero $\bm{y} \in \mathbb{R}^N$.

Suppose $\bm{x}$ and $\bm{y}$ are linearly independent (with probability $1$), then $\nexists\; c \neq 0$ that $\bm{x} = c\bm{y}$. If $\exists\; d \neq 0$ that ReLU$(\bm{x})$= $d$ ReLU$(\bm{y})$, then $\bm{x_+} = d \bm{y_+} $ and $\bm{x_-} \neq d \bm{y_-}$ and the existence of this kind of $\bm{x},\bm{y}$ has a nonzero probability $\frac{1}{2^N} \frac{1+N}{2^N}$ (See Lemma 3); other than these cases, the independency will be kept after the ReLU transformation. %This means that point-wise ReLU weakens the linear independency between vectors with probability 1.

Suppose $y$ is linearly dependent of $\bm{x}$ (with probability 0), \ie{}, $\exists\; c \neq 0$ that $\bm{x} = c \bm{y}$. If $c>0$, we have $\bm{x_-} = c \bm{y_-}$ and $ \bm{x_+} = c \bm{y_+}$. Since ReLU$(\bm{x})$=$\bm{x_+}$, ReLU$(\bm{y})$=$\bm{y_+}$, then ReLU$(\bm{x})$= $c$ReLU$(\bm{y})$. If $c<0$, we have $\bm{x_-} = c \bm{y_+}$ and $ \bm{x_+} = c \bm{y_-}$. If $\exists\; d \neq 0$ that ReLU$(\bm{x})$= $d$ ReLU$(\bm{y})$, this means $\bm{x_+} = d \bm{y_+} = \frac{d}{c} \bm{x_-}$. This $d$ exists when $\bm{x_+} $ and $\bm{x_-}$ are linearly dependent, which only holds when $x=0$. This happens with probability 0. Then, ReLU$(\bm{x})$ and $d$ ReLU$(\bm{y})$ will be independent. So whether ReLU keep the dependency between to vectors or not depends on the sign of $c$. Thus, under the assumption, we have probability $\frac{1}{2}$ that ReLU keeps the dependency.

According to the discussion above, we have
%Unless $\bm{x_-} \neq 0$ and $\bm{x_+} \neq 0$, the linear dependency still holds. The overall probability of losing dependency between column vectors is 0. then this case is trivial.% since it happens with probability 0.
\begin{align*}
&\doubleP\left(\text{rank}\left(\text{ReLU}([\bm{x}, \bm{y}])\right) \leq \text{rank}([\bm{x}, \bm{y}]) \;|\;\bm{x}, \bm{y} \in \mathbb{R}^N\right) = \frac{\doubleP\left(\text{rank}\left(\text{ReLU}([\bm{x}, \bm{y}])\right) \leq \text{rank}([\bm{x}, \bm{y}]) ,\bm{x}, \bm{y} \in \mathbb{R}^N\right)}{\doubleP\left(\bm{x}, \bm{y} \in \mathbb{R}^N\right)}\\
& = \doubleP\left(\text{rank}\left(\text{ReLU}([\bm{x}, \bm{y}])\right) \leq \text{rank}([\bm{x}, \bm{y}]) \;|\; \text{rank}([\bm{x}, \bm{y}]) = 1,\bm{x}, \bm{y} \in \mathbb{R}^N \right) \doubleP\left(\text{rank}([\bm{x}, \bm{y}]) = 1,\bm{x}, \bm{y} \in \mathbb{R}^N \right) +\\
&\doubleP\left(\text{rank}\left(\text{ReLU}([\bm{x}, \bm{y}])\right) \leq \text{rank}([\bm{x}, \bm{y}]) \;|\; \text{rank}([\bm{x}, \bm{y}]) = 2,\bm{x}, \bm{y} \in \mathbb{R}^N \right) \doubleP\left(\text{rank}([\bm{x}, \bm{y}]) = 2,\bm{x}, \bm{y} \in \mathbb{R}^N \right)\\
& = 0 \times \frac{1}{2} + 1 \times 1 = 1.
\end{align*}
Lemma proved.
\end{proof}
\begin{lemma} 3
 Suppose we randomly sample $\bm{x}, \bm{y} \in \mathbb{R}^N$ under a continuous distribution, then
 $$\doubleP(\bm{x_+} = d \bm{y_+}, \bm{x_-} \neq d \bm{y_-} \;|\; x,y \in \mathbb{R}^N, d \neq 0) = \frac{1}{2^N} \frac{1+N}{2^N}$$
\end{lemma}
\begin{proof}
\begin{align*}
    &\doubleP(\bm{x_+} = d \bm{y_+}, \bm{x_-} \neq d \bm{y_-}) \\
    & = P (\bm{x_+} = d \bm{y_+}, \bm{x_-} \neq d \bm{y_-} \;|\;  df(\bm{x_+}) \leq 1 ) \; P(df(\bm{x_+}) \leq 1) + P (\bm{x_+} = d \bm{y_+}, \bm{x_-} \neq d \bm{y_-} \;|\;  df(\bm{x_+}) > 1 ) \; P (df(\bm{x_+}) > 1)\\
    &=\frac{1}{2^N} \frac{1+N}{2^N} + 0 \cdot \frac{2^N -1 -N}{2^N} =\frac{1}{2^N} \frac{1+N}{2^N}
\end{align*}
where $df$ denotes degree of freedom. $df(\bm{x_+}) \leq 1$ means that $\bm{x}$ can at most have one dimension to be positive and there are $1+N$ out of $2^N$ hyperoctants that satisfies this condition. The set of $\bm{y}$ that can make $\bm{x_+} = d \bm{y_+}, \bm{x_-} \neq d \bm{y_-}$ hold has an area of $\frac{1}{2^N}$, \ie{} when $\bm{y}$ is in the same hyperoctant as $\bm{x}$. If $\bm{x}$ lies in other hyperoctants, $df(y_-) \leq N-2$. And since $\bm{x_+} = d \bm{y_+}$, $\bm{y}$ is just a low dimensional surface in $\mathbb{R}^N$ with area 0.
\end{proof}
\begin{theorem} 1
Suppose that $\mathcal{G}$ has $k$ connected components and the diffusion operator $L$ is defined as that in \eqref{eq0}. Let $\bm{X} \mathbb{R}^{N \times F}$ be any block vector that randomly sampled under a continuous distribution and $\{W_0, W_1, \dots, W_n\}$ be any set of parameter matrices, if $\mathcal{G}$ has no bipartite components, then in \eqref{eq1}, as $n \to \infty$, $\text{rank}(\bm{Y'}) \leq k$ almost surely.
%, consider the space $\mathbb{R}^{N \times F}$ of all block vectors $\bm{X}$ and any set of parameter matrices $\{W_0, W_1, \dots, W_n\}$, we have $\text{rank}(\bm{Y'}) \leq k$ with probability 1, as $n\rightarrow \infty$.
\end{theorem}	

\begin{proof}
Upon the conclusions in Lemma 2-3, we have rank$\left(\text{ReLU} (L \textbf{X})\right) \leq$rank($L \textbf{X}$) with probability 1 and it is obvious rank$(L \bm{X} W_0) \leq$rank($L \bm{X}$). Using these two inequality iteratively for \eqref{eq1}, we have rank($\bm{Y'}$) $\leq$ rank($L^{n+1}  \bm{X}$). Based on Lemma 1, we have probability 1 to get
\begin{equation*} \label{eq3}
\underset{n \rightarrow \infty }{\text{lim}} \; \text{rank}(\bm{Y'}) \leq \underset{n \rightarrow \infty }{\text{lim}}\; \text{rank}(L^{n+1} \bm{X}) = \text{rank} ([\bm{v_1},\dots, \bm{v_k}] [\bm{\theta_1},\dots, \bm{\theta_F}]) \leq \text{rank}([\bm{v_1},\dots, \bm{v_k}]) = k,
\end{equation*}
where $\bm{\theta_i} \in \mathbb{R}^k, i=1,\dots,F$. Thus, rank($\bm{Y'}$) $\leq k$
\end{proof}

%\begin{lemma} 4
 %Given $\bm{x}, y\in \mathbb{R}^2$ and point-wise function $\text{Tanh}(x) = \frac{e^x - e^{-x}}{e^x + e^{-x}}$, we have
%$$\doubleP(\text{rank}\left(\text{Tanh}([x,y])\right) \geq \text{rank}([x,y]) \;|\; x,y \in \mathbb{R}^N) = 1 $$
%\end{lemma}

\begin{theorem} 2
Suppose we randomly sample $\bm{x}, \bm{y} \in \mathbb{R}^N$ under a continuous distribution and we have the point-wise function $\text{Tanh}(z) = \frac{e^z - e^{-z}}{e^z + e^{-z}}$, we have
$$\doubleP(\text{rank}\left(\text{Tanh}([\bm{x},\bm{y}])\right) \geq \text{rank}([\bm{x},\bm{y}]) \;|\; \bm{x},\bm{y} \in \mathbb{R}^N) = 1$$
\end{theorem}

\begin{proof}
We first prove it for $N=2$. It is trivial when $x,y$ have 0 elements.

Suppose $\bm{x} = [x_1, x_2] ,\bm{y} = [y_1, y_2]$ are linearly dependent and all of their elements are nonzero. If $\text{Tanh}(\bm{x})$ and $\text{Tanh}(\bm{y})$ are still linearly dependent, we must have
$$ \frac{\text{Tanh}(x_1)}{\text{Tanh}(y_1)} = \frac{\text{Tanh}(x_2)}{\text{Tanh}(y_2)}, \text{ \st{} } \frac{x_1}{y_1} = \frac{x_2}{y_2} $$
This equation only have one solution
$$x_1 = \frac{1033977}{9530},\; x_2 = -\frac{929}{10} ,\; y_1 = -\frac{1113}{10}, \; y_2 = \frac{953}{10}$$
which means point-wise $\text{Tanh}$ transformation will break the dependency of $\bm{x},\bm{y}$ with probability 1. %their linear dependency will be kept only when $x=y$, which is trivial.

If $x,y$ are linearly independent, suppose $\text{Tanh}(\bm{x}) = \bm{x'} = [x_1', x_2']$ is a vector in $\mathbb{R}^2$, and the set of vectors in $\mathbb{R}^2$ that can be transformed by point-wise Tanh to the same line as $\bm{x'}$ covers an area of probability 0. That is for any fixed $\bm{x'} \in \mathbb{R}^2$, the solution of
$$\frac{\text{Tanh}(y_1)}{\text{Tanh}(y_2)} = \frac{x_1'}{x_2'}$$
covers an area of 0. Thus,
$$ \doubleP(\text{rank}\left(\text{Tanh}([\bm{x},\bm{y}])\right) \geq \text{rank}([\bm{x},\bm{y}]) \;|\; \bm{x},\bm{y} \in \mathbb{R}^N) = 1 $$
which means point-wise Tanh transformation will increase the independency between vectors in $\mathbb{R}^2$.

Similar to $\mathbb{R}^2$, suppose $\bm{x} = [x_1, x_2, \dots, x_N] ,\bm{y} = [y_1, y_2, \dots, y_N]$ are linearly dependent, if all elements in $\bm{x}, \bm{y}$ are nonzero and $\text{Tanh}(\bm{x})$ and $\text{Tanh}(\bm{y})$ are still linearly dependent, we must have
$$ \frac{\text{Tanh}(x_1)}{\text{Tanh}(y_1)} = \frac{\text{Tanh}(x_2)}{\text{Tanh}(y_2)} = \cdots = \frac{\text{Tanh}(x_N)}{\text{Tanh}(y_N)}, \text{ \st{} } \frac{x_1}{y_1} = \frac{x_2}{y_2} \cdots = \frac{x_N}{y_N} $$
The solution of any pair of equations covers an area of probability 0 in $\mathbb{R}^N$. Actually, this still holds when $\bm{x},\bm{y}$ have some 0 elements, \ie{} for any subset of the above equations, the area of the solution is still 0.

If $\bm{x},\bm{y}$ are linearly independent, suppose $\text{Tanh}(\bm{x}) = \bm{x'}$ is a vector in $\mathbb{R}^N$, and the space in $\mathbb{R}^N$ that can be transformed by point-wise $\text{Tanh}$ to the same line as $\bm{x'}$ covers an area of probability 0. Therefore, Lemma 2 still holds in $\mathbb{R}^N$.
\end{proof}

\subsection*{Appendix \rom{2}: Numerical Experiments on Synthetic Data}
\label{appendix:2}
The goal of the experiments is to test which network structure with which kind of activation function has the potential to be extended to deep architecture. We measure this potential by the numerical rank of the output features in each hidden layer of the networks using synthetic data. The reason of choosing this measure can be explained by Theorem \ref{thm1}. We build the certain networks with depth 100 and the data is generated as follows.

We first randomly generate edges of an Erd\H{o}s-R\'enyi graph $G(1000,0.01)$, \ie{} the existence of the edge between any pair of nodes is a Bernoulli random variable with $p = 0.01$.  Then, we construct the corresponding adjacency matrix $A$ of the graph which is a $\mathbb{R}^{1000 \times 1000} $ matrix. We generate a $\mathbb{R}^{1000 \times 500}$ feature matrix $\bm{X}$ and each of its element is drawn from $N(0,1)$. We normalize $A$ and $\bm{X}$ as \cite{kipf2016classification} and abuse the notation $A,\bm{X}$ to denote the normalized matrices. We keep 3 blocks in each layer of truncated block Krylov network. The number of input channel in each layer depends on the network structures and the number of output channel is set to be 128 for all networks. Each element in every parameter matrix $W_i, \; i = 1, \dots, 100$ is randomly sampled from $N(0,1)$ and the size is $\mathbb{R}^{\# input \times \#output}$. With the synthetic $A,\bm{X},W_i$, we simulate the feedforward process according to the network architecture and collect the numerical rank (at most 128) of the output in each of the 100 hidden layers. For each activation function under each network architecture, we repeat the experiments for 20 times and plot the mean results with standard deviation bars.
\subsection*{Appendix \rom{3}: Rank Comparison of Activation Functions and Networks}
\label{appendix:3}
\begin{figure*}[htbp]
\centering
\subfloat[GCN]{
\captionsetup{justification = centering}
\includegraphics[width=0.32\textwidth]{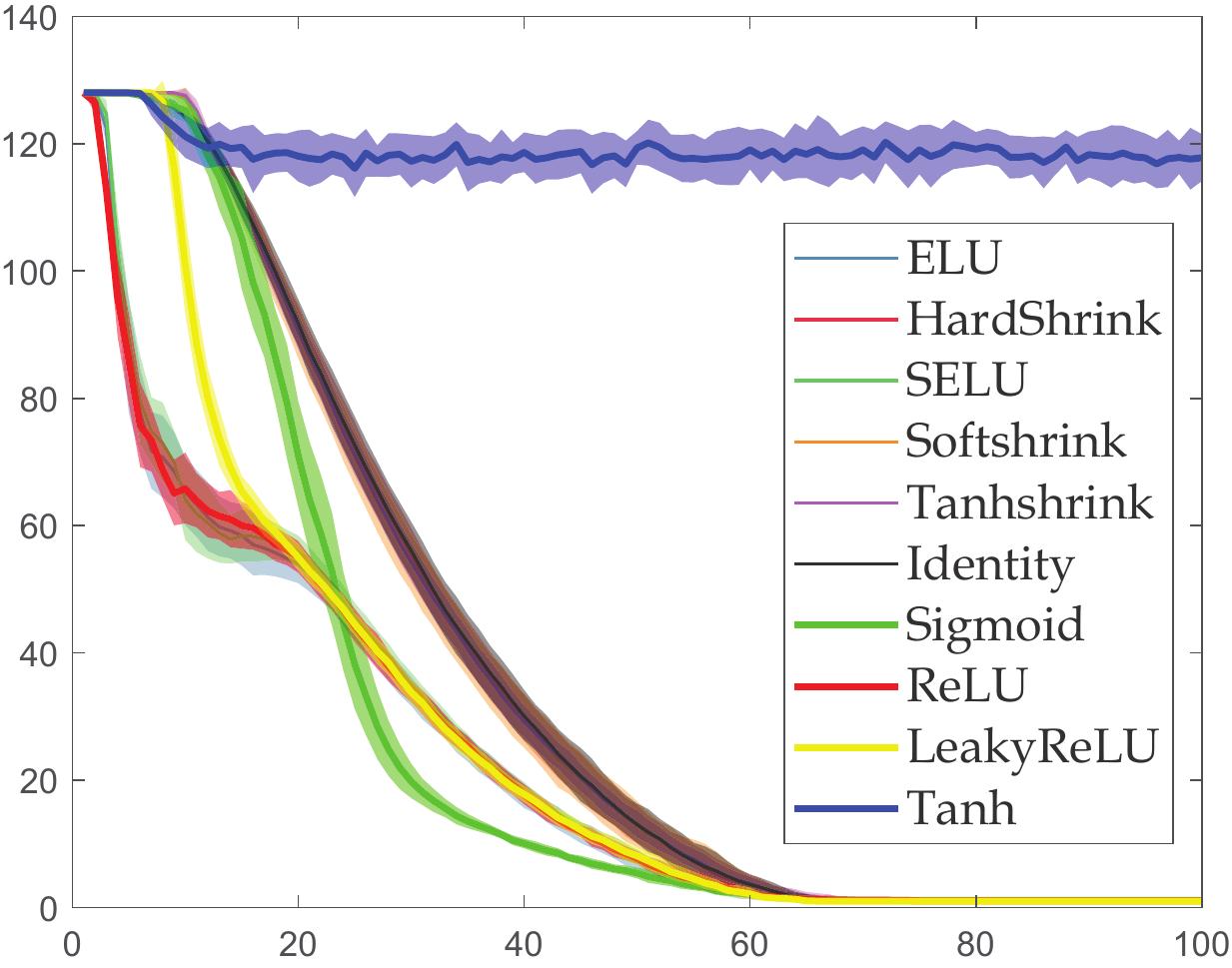}}
\hfill
\subfloat[Snowball]{
\captionsetup{justification = centering}
\includegraphics[width=0.32\textwidth]{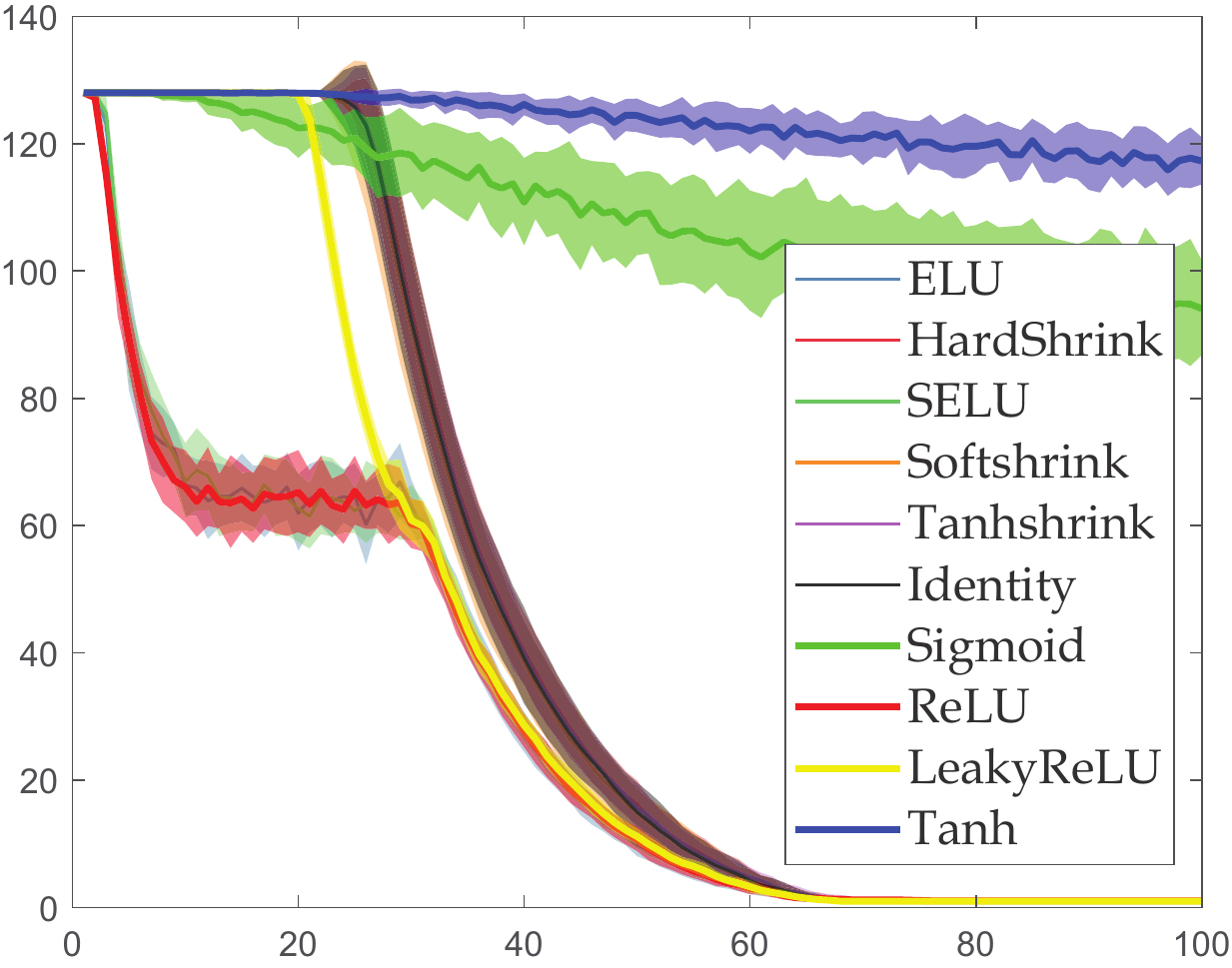}}
\hfill
\subfloat[Truncated Block Krylov]{
\captionsetup{justification = centering}
\includegraphics[width=0.32\textwidth]{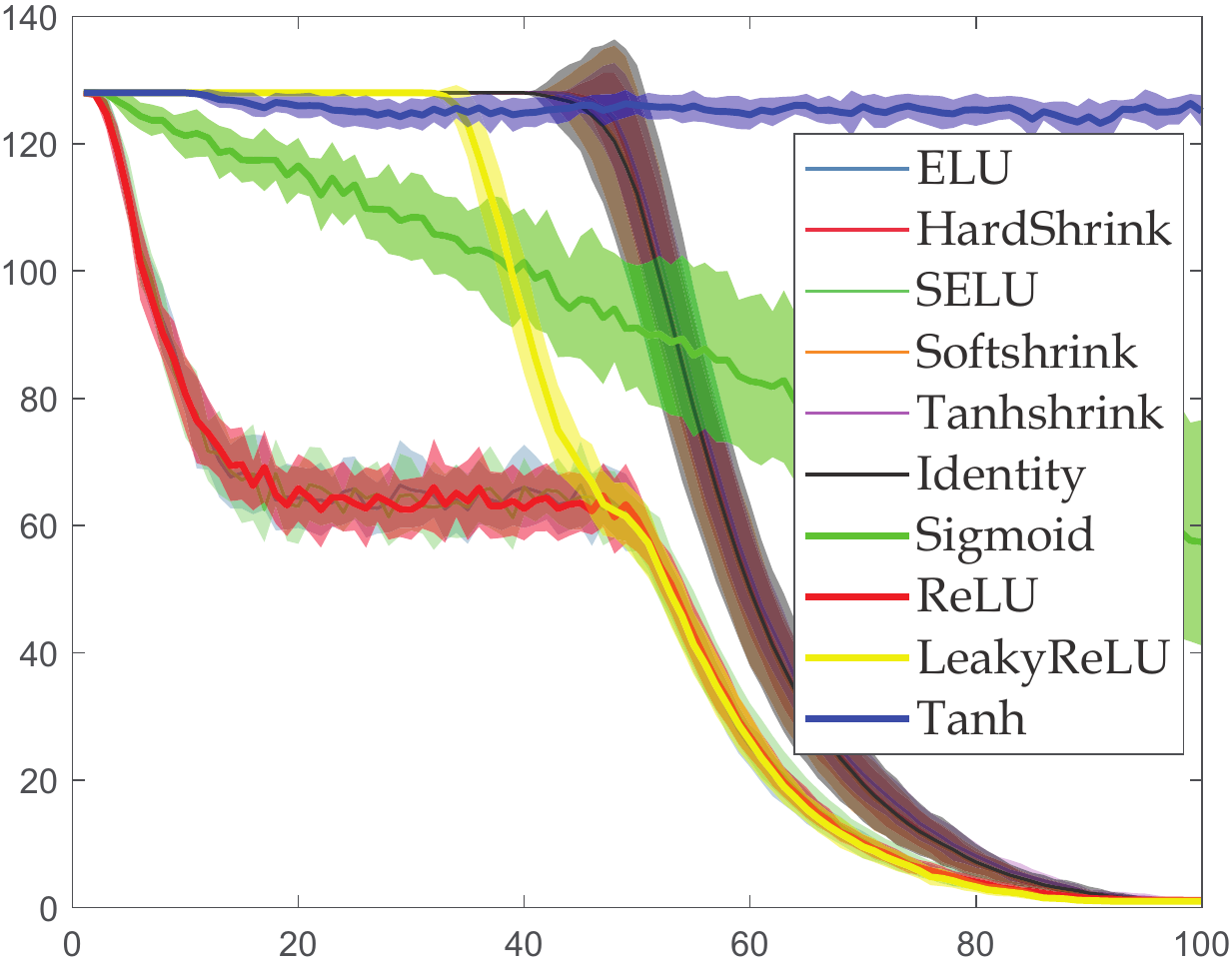}}
\caption{Column ranks of different activation functions with the same architecture}
\end{figure*}

\begin{figure*}[htbp]
\centering
\subfloat[ReLU]{
\captionsetup{justification = centering}
\includegraphics[width=0.32\textwidth]{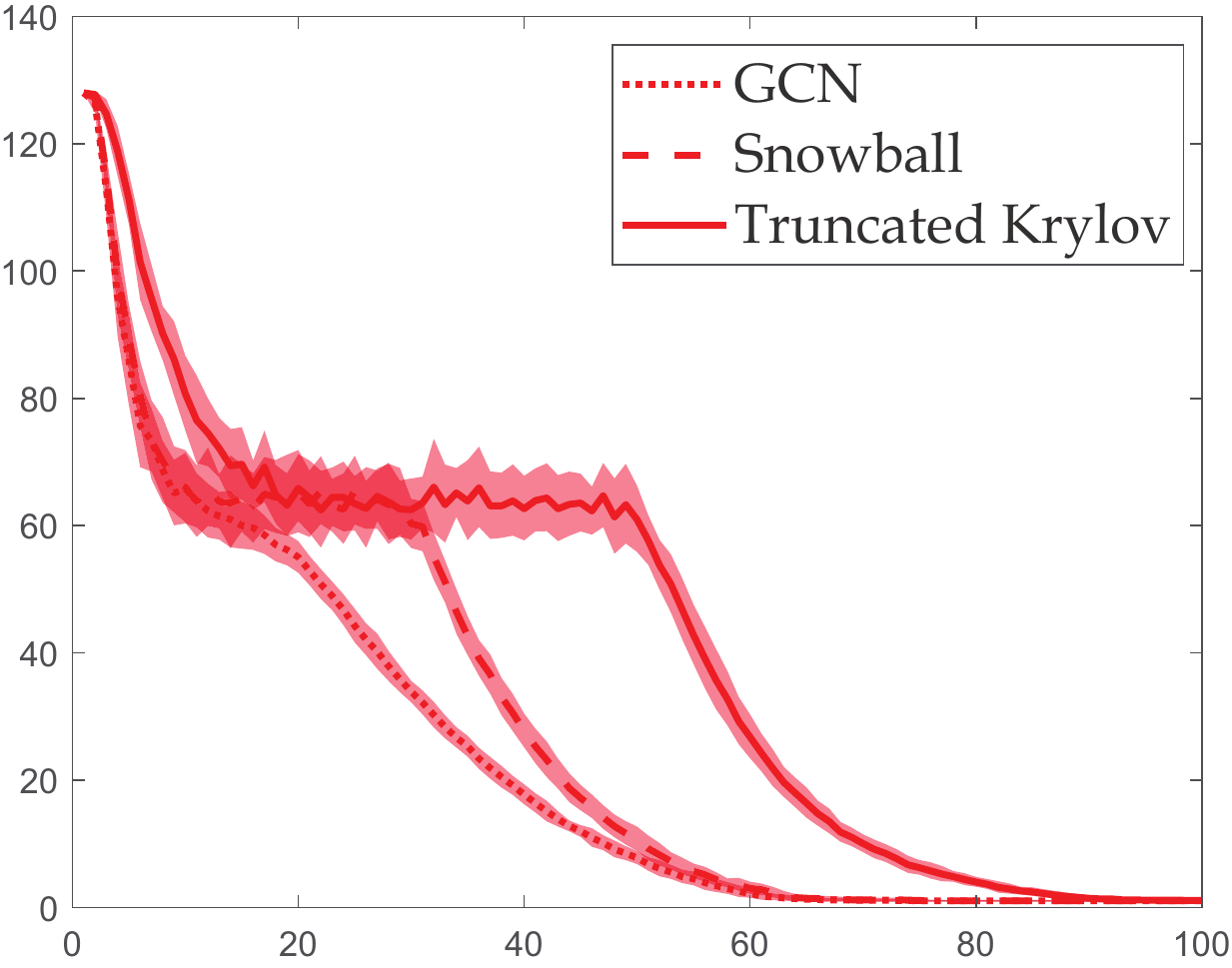}}
\hfill
\subfloat[Identity]{
\captionsetup{justification = centering}
\includegraphics[width=0.32\textwidth]{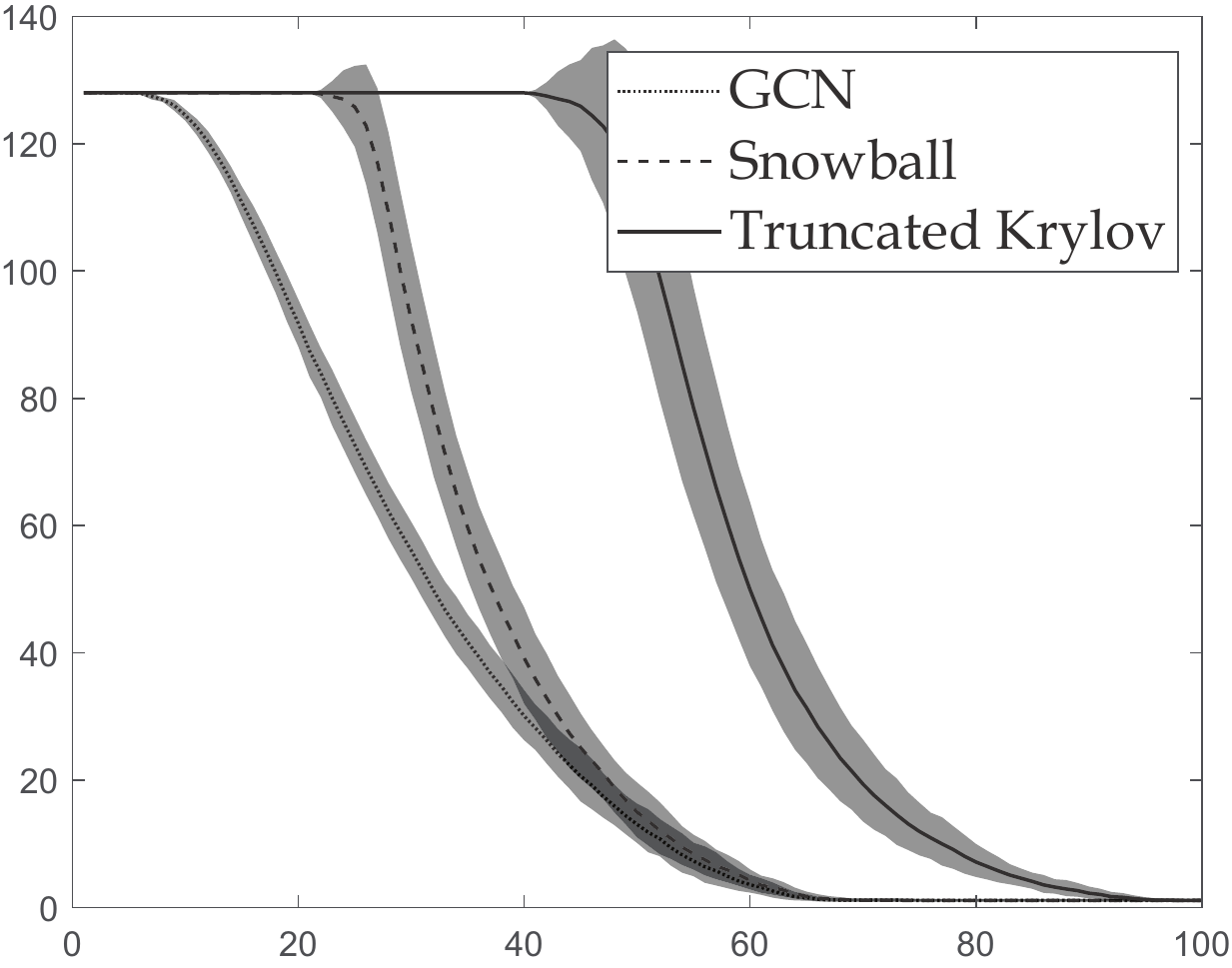}}
\hfill
\subfloat[Tanh]{
\captionsetup{justification = centering}
\includegraphics[width=0.32\textwidth]{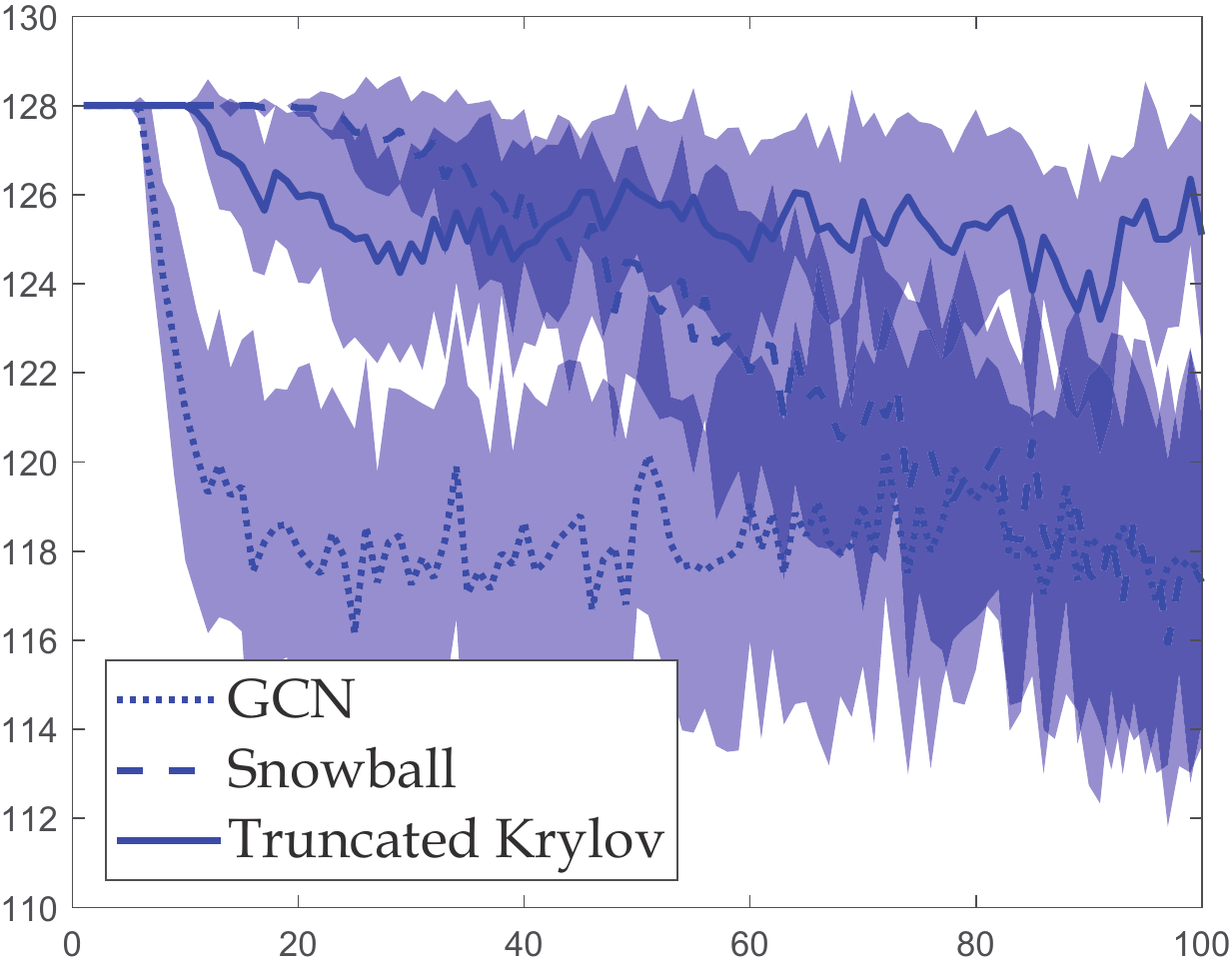}}
\caption{Column ranks of different architectures with the same activation function}
\end{figure*}

\subsection*{Appendix \rom{4}: Spectrum of the Datasets}
\label{appendix:4}
%See figure \ref{spectrum_dataset}.
\label{spectrum}
\begin{figure*}[bhtp]
\centering
\subfloat[Cora]{
\captionsetup{justification = centering}
\includegraphics[width=0.32\textwidth]{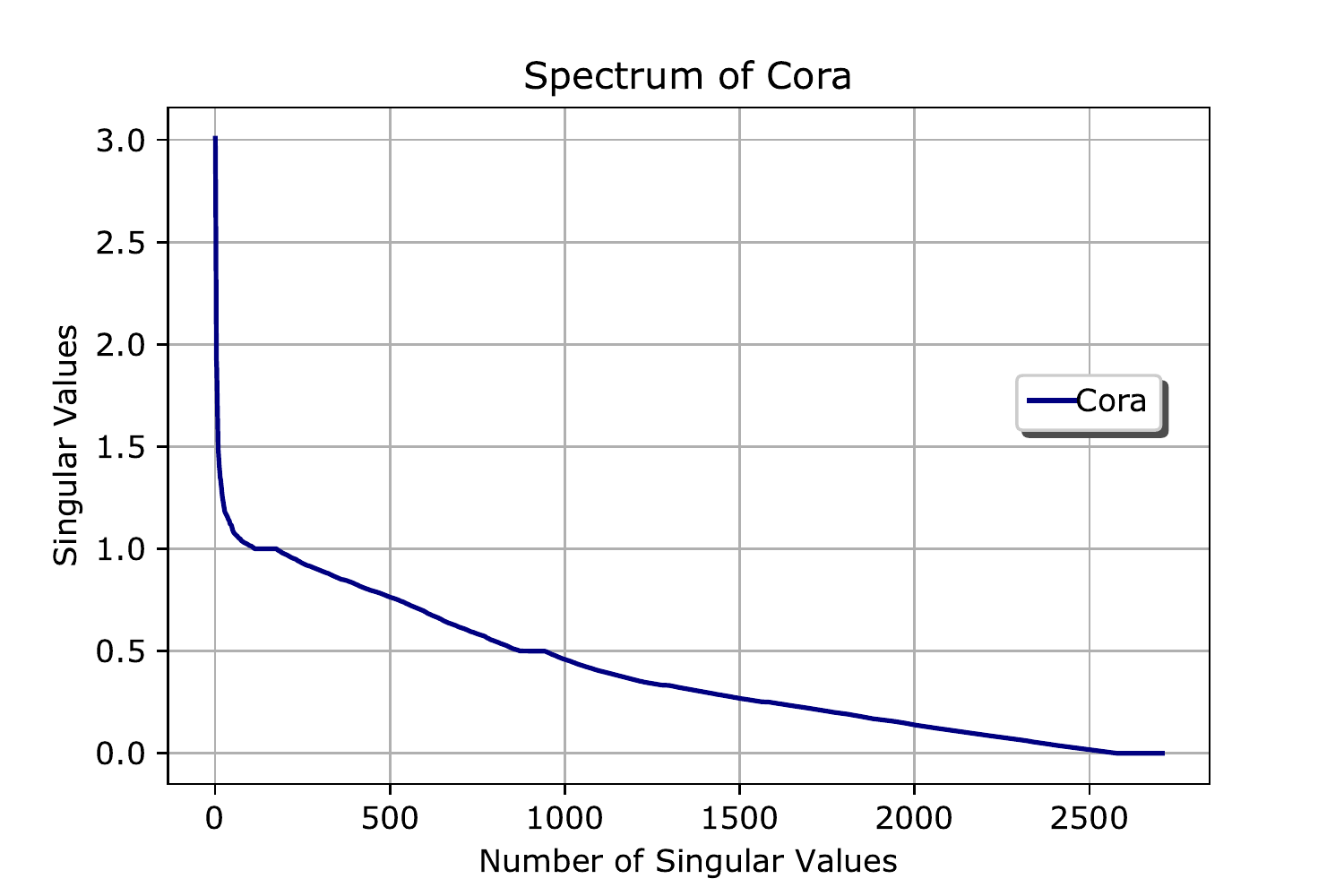}}
\hfill
\subfloat[Citeseer]{
\captionsetup{justification = centering}
\includegraphics[width=0.32\textwidth]{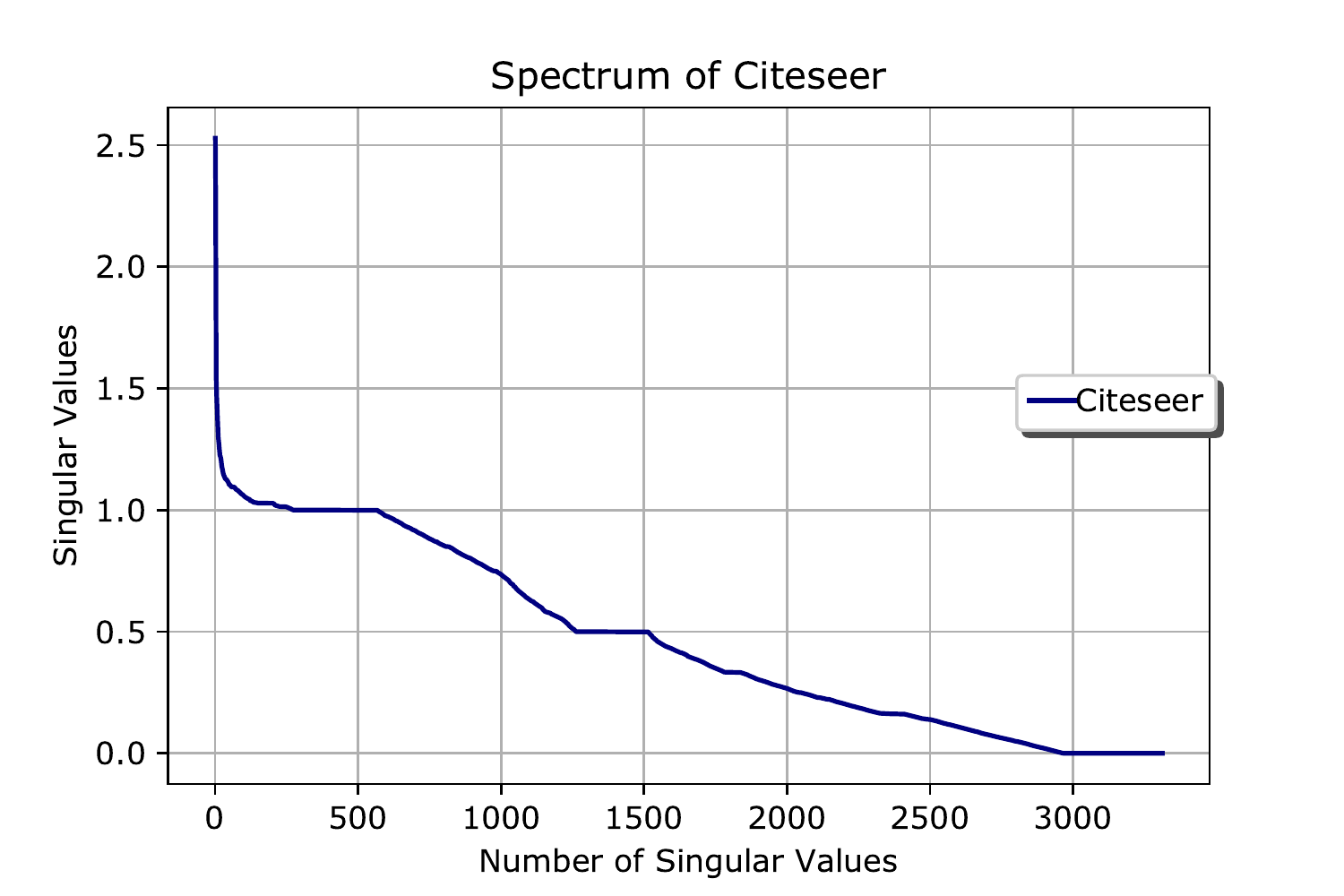}}
\hfill
\subfloat[PubMed]{
\captionsetup{justification = centering}
\includegraphics[width=0.32\textwidth]{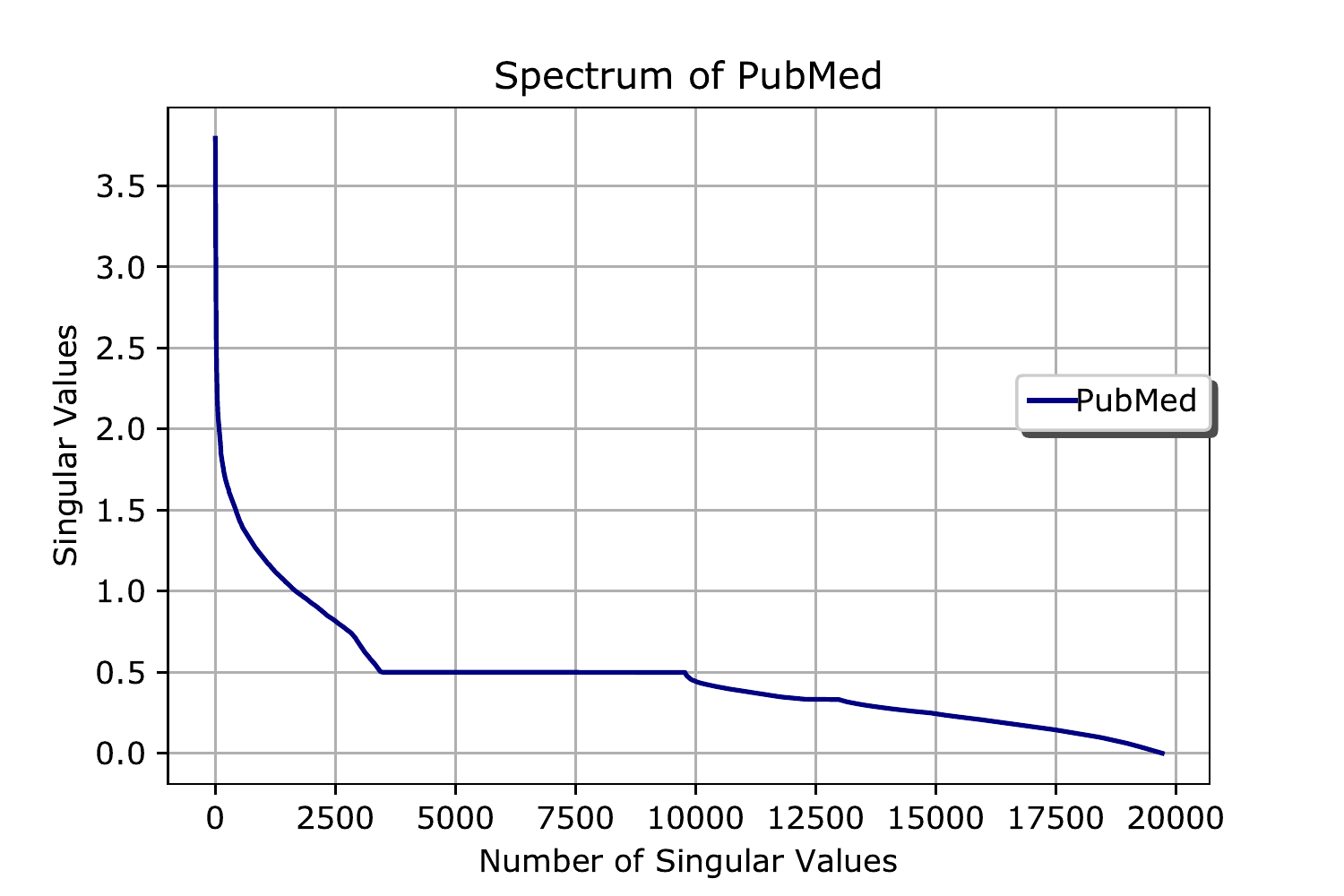}}
\caption{Spectrum of the renormalized adjacency matrices for several datasets}
\label{spectrum_dataset}
\end{figure*}

%\subsection{Appendix \rom{7}: Adaptive Depth}
%In practice, it is found that the performance of the training model is sensitive to the position of the training data in graph. If many of the training nodes happen to be located at some pivotal positions, \eg{} the center of a densely connected cluster, it is much faster to diffuse the feature signals across the graph. Otherwise, if a large part of our training nodes are isolated nodes or have few links with others, it will take much more steps for the signal to diffuse to other parts of the graph. Therefore, before we train the model, we should take an extra step to decide how many diffusion steps we need, \ie{} the number of layers of our network. If the depth is too shallow, the signal cannot diffuse to some unlabeled nodes. On the other hand, if the network is too deep, we have overfitting problem and our prediction are likely to be interfered by some noise signals from some pretty far nodes. So we try to design a method that the training data can deliver enough message to reachable nodes but the reachable nodes do not receive too much redundant message.
%Another way is we compare how many new nodes that can be reach and how many new edges added to the reached nodes.
%By using adaptive layer numbers, the final training results become more stable but there still exists some cases that you cannot every get a good upper bound. For this reason, we try to sacrifice the efficiency and directly use a large Dense GCN for any case.

\subsection*{Appendix \rom{5}: Experiment Settings and Hyperparameters}
\label{appendix:5}
The so-called public splits in \cite{liao2019lanczos} and the setting that randomly sample 20 instances for each class as labeled data in \cite{yang2016revisiting} is actually the same. Most of the results for the algorithms with validation are cited from \cite{liao2019lanczos}, where they are reproduced with validation. However, some of them actually do not use validation in original papers and can achieve better results. In the paper, We compare with their best results.

We use NVIDIA apex amp mixed-precision plugin for PyTorch to accelerate our experiments. Most of the results were obtained from NVIDIA V100 clusters on Beluga of Compute-Canada, with minor part of them obtained from NVIDIA K20, K80 clusters on Helios Compute-Canada. The hyperparameters are searched using Bayesian optimization.

A useful tip is the smaller your training set is, the larger dropout probability should be set and the larger early stopping you should have.

Table \ref{tab:hyperparameters_without_validation} and Table \ref{tab:hyperparameters_with_validation} shows the hyperparameters to achieve the performance in the experiments, for cases without and with validation, respectively. When conducting the hyperparameter search, we encounter memory problems: current GPUs cannot afford deeper and wider structures. But we do observe better performance with the increment of the network size. It is expected to achieve better performance with more advanced deep learning devices.

% Table generated by Excel2LaTeX from sheet 'Overall'
\begin{table}[htbp]
\setlength{\tabcolsep}{1.5pt}
  \centering
  \caption{Hyperparameters for Tests with Validation}
  \scriptsize
    \begin{tabular}{ccccc|cccccc}
    \toprule
    \toprule
    \multirow{2}[1]{*}{Architecture} & \multirow{2}[1]{*}{Dataset} & \multirow{2}[1]{*}{Percentage} & \multicolumn{2}{c|}{Accuracy} & \multicolumn{6}{c}{Corresponding Hyperparameters} \\
          &       &       & Ours  & SOTA  & lr    & weight\_decay & hidden & layers/n\_blocks & dropout & optimizer \\
    \midrule
    \multirow{11}[0]{*}{\textbf{linear Snowball}} & \multirow{4}[0]{*}{\textbf{Cora}} & 0.5\% & 69.99 & 60.8  & 1.0689E-03 & 1.4759E-02 & 128   & 6     & 0.66987 & RMSprop \\
          &       & 1\%   & 73.10 & 67.5  & 1.4795E-03 & 2.3764E-02 & 128   & 9     & 0.64394 & RMSprop \\
          &       & 3\%   & 80.96 & 77.7  & 2.6847E-03 & 5.1442E-03 & 64    & 9     & 0.23648 & RMSprop \\
          &       & 5.2\% (public) & 83.19 & 83.0  & 1.6577E-04 & 1.8606E-02 & 1024  & 3     & 0.65277 & RMSprop \\
          & \multirow{3}[0]{*}{\textbf{CiteSeer}} & 0.5\% & 59.41 & 53.8  & 4.9284E-04 & 6.9420E-03 & 512   & 11    & 0.90071 & RMSprop \\
          &       & 1\%   & 65.85 & 63.3  & 3.2628E-03 & 1.6374E-02 & 512   & 3     & 0.97331 & RMSprop \\
          &       & 3.6\% (public) & 73.54 & 72.5  & 2.8218E-03 & 1.9812E-02 & 5000  & 1     & 0.98327 & Adam \\
          & \multirow{4}[0]{*}{\textbf{Pubmed}} & 0.03\% & 68.12 & 61.0  & 2.1124E-03 & 4.4161E-02 & 128   & 7     & 0.78683 & RMSprop \\
          &       & 0.05\% & 70.04 & 68.8  & 4.9982E-03 & 2.6460E-02 & 128   & 4     & 0.86788 & RMSprop \\
          &       & 0.1\% & 73.83 & 73.4  & 1.2462E-03 & 4.9303E-02 & 128   & 6     & 0.3299 & RMSprop \\
          &       & 0.3\% (public) & 79.23 & 79.0  & 2.4044E-03 & 2.3157E-02 & 4000  & 1     & 0.98842 & Adam \\
    \midrule
    \multirow{11}[1]{*}{\textbf{Snowball}} & \multirow{4}[0]{*}{\textbf{Cora}} & 0.5\% & 72.96 & 60.8  & 2.3228E-04 & 2.1310E-02 & 950   & 7     & 0.88945 & RMSprop \\
          &       & 1\%   & 76.76 & 67.5  & 1.5483E-04 & 1.3963E-02 & 250   & 15    & 0.55385 & RMSprop \\
          &       & 3\%   & 80.72 & 77.7  & 1.6772E-03 & 1.0725E-02 & 64    & 14    & 0.80611 & RMSprop \\
          &       & 5.2\% (public) & 83.60 & 83.0  & 1.2994E-05 & 9.4469E-03 & 5000  & 3     & 0.025052 & RMSprop \\
          & \multirow{3}[0]{*}{\textbf{CiteSeer}} & 0.5\% & 62.05 & 53.8  & 2.0055E-03 & 3.1340E-02 & 512   & 5     & 0.88866 & RMSprop \\
          &       & 1\%   & 64.23 & 63.3  & 1.8759E-03 & 9.3636E-03 & 128   & 7     & 0.77334 & RMSprop \\
          &       & 3.6\% (public) & 72.61 & 72.5  & 2.5527E-03 & 6.2812E-03 & 256   & 1     & 0.56755 & RMSprop \\
          & \multirow{4}[1]{*}{\textbf{Pubmed}} & 0.03\% & 70.78 & 61.0  & 1.1029E-03 & 1.8661E-02 & 100   & 15    & 0.83381 & RMSprop \\
          &       & 0.05\% & 73.23 & 68.8  & 3.7159E-03 & 2.2088E-02 & 400   & 9     & 0.9158 & RMSprop \\
          &       & 0.1\% & 76.52 & 73.4  & 4.9106E-03 & 3.0777E-02 & 100   & 15    & 0.79133 & RMSprop \\
          &       & 0.3\% (public) & 79.54 & 79.0  & 4.9867E-03 & 3.5816E-03 & 3550  & 1     & 0.98968 & Adam \\
    \midrule
    \multirow{11}[2]{*}{\textbf{truncated Krylov}} & \multirow{4}[1]{*}{\textbf{Cora}} & 0.5\% & 73.89 & 60.8  & 1.6552E-04 & 4.4330E-02 & 4950  & 27    & 0.97726 & Adam \\
          &       & 1\%   & 77.38 & 67.5  & 2.8845E-04 & 4.8469E-02 & 4950  & 30    & 0.93928 & Adam \\
          &       & 3\%   & 82.23 & 77.7  & 8.6406E-04 & 4.0126E-03 & 2950  & 16    & 0.98759 & Adam \\
          &       & 5.2\% (public) & 83.51 & 83.0  & 1.0922E-03 & 3.5966E-02 & 1950  & 10    & 0.98403 & Adam \\
          & \multirow{3}[0]{*}{\textbf{CiteSeer}} & 0.5\% & 63.65 & 53.8  & 2.8208E-03 & 4.3395E-02 & 1150  & 30    & 0.92821 & Adam \\
          &       & 1\%   & 68.36 & 63.3  & 3.9898E-03 & 3.8525E-03 & 100   & 27    & 0.71951 & Adam \\
          &       & 3.6\% (public) & 73.89 & 72.5  & 1.8292E-03 & 4.2295E-02 & 600   & 11    & 0.98865 & Adam \\
          & \multirow{4}[1]{*}{\textbf{Pubmed}} & 0.03\% & 71.11 & 61.0  & 3.6759E-03 & 1.2628E-02 & 512   & 8     & 0.95902 & RMSprop \\
          &       & 0.05\% & 72.86 & 68.8  & 4.0135E-03 & 4.8831E-02 & 4250  & 5     & 0.95911 & Adam \\
          &       & 0.1\% & 75.68 & 73.4  & 4.7562E-03 & 3.7134E-02 & 950   & 7     & 0.96569 & Adam \\
          &       & 0.3\% (public) & 79.88 & 79.0  & 3.9673E-04 & 2.2931E-02 & 1900  & 4     & 0.000127 & Adam \\
    \bottomrule
    \bottomrule
    \end{tabular}%
  \label{tab:hyperparameters_with_validation}%
\end{table}%

% Table generated by Excel2LaTeX from sheet 'Overall'
\begin{table}[htbp]
\setlength{\tabcolsep}{1.5pt}
  \centering
  \caption{Hyperparameters for Tests without Validation}
  \scriptsize
    \begin{tabular}{ccccc|cccccc}
    \toprule
    \toprule
    \multirow{2}[2]{*}{Architecture} & \multirow{2}[2]{*}{Dataset} & \multirow{2}[2]{*}{Percentage} & \multicolumn{2}{c|}{Accuracy} & \multicolumn{6}{c}{Correspondong Hyperparameters} \\
          &       &       & Ours  & SOTA  & lr    & weight\_decay & hidden & layers/n\_blocks & dropout & Optimizer \\
    \midrule
    \multirow{16}[2]{*}{\textbf{linear Snowball}} & \multirow{6}[1]{*}{\textbf{Cora}} & 0.5\% & 69.53 & 61.5  & 4.4438E-05 & 1.7409E-02 & 550   & 12    & 0.007753 & Adam \\
          &       & 1\%   & 74.12 & 69.9  & 1.0826E-03 & 3.3462E-03 & 1250  & 3     & 0.50426 & Adam \\
          &       & 2\%   & 79.43 & 75.9  & 2.4594E-06 & 9.6734E-03 & 1650  & 12    & 0.34073 & Adam \\
          &       & 3\%   & 80.41 & 78.5  & 2.8597E-05 & 3.4732E-02 & 900   & 15    & 0.039034 & Adam \\
          &       & 4\%   & 81.3  & 80.4  & 3.6830E-05 & 1.5664E-02 & 3750  & 4     & 0.93797 & Adam \\
          &       & 5\%   & 82.19 & 81.7  & 5.8323E-06 & 8.5940E-03 & 2850  & 5     & 0.14701 & Adam \\
          & \multirow{6}[0]{*}{\textbf{CiteSeer}} & 0.5\% & 56.76 & 56.1  & 4.5629E-03 & 2.0106E-03 & 300   & 3     & 0.038225 & Adam \\
          &       & 1\%   & 65.44 & 62.1  & 3.5530E-05 & 4.9935E-02 & 600   & 6     & 0.03556 & Adam \\
          &       & 2\%   & 68.78 & 68.6  & 6.1176E-06 & 3.0101E-02 & 1950  & 3     & 0.040484 & Adam \\
          &       & 3\%   & 71    & 70.3  & 2.1956E-05 & 4.3569E-02 & 3350  & 3     & 0.30207 & Adam \\
          &       & 4\%   & 72.23 & 70.8  & 9.1952E-05 & 4.6407E-02 & 3350  & 2     & 0.018231 & Adam \\
          &       & 5\%   & 72.21 & 71.3  & 3.7173E-03 & 1.9605E-03 & 2950  & 1     & 0.96958 & Adam \\
          & \multirow{4}[1]{*}{\textbf{Pubmed}} & 0.03\% & 64.133 & 62.2  & 1.0724E-03 & 8.1097E-03 & 64    & 4     & 0.8022 & RMSProp \\
          &       & 0.05\% & 69.48 & 68.3  & 1.5936E-03 & 3.0236E-03 & 6     & 10    & 0.73067 & RMSProp \\
          &       & 0.1\% & 72.93 & 72.7  & 4.9733E-03 & 1.3744E-03 & 128   & 3     & 0.91214 & RMSProp \\
          &       & 0.3\% & 79.33 & 79.2  & 1.7998E-03 & 9.6753E-04 & 512   & 1     & 0.97483 & RMSProp \\
    \midrule
    \multirow{16}[2]{*}{\textbf{Snowball}} & \multirow{6}[1]{*}{\textbf{Cora}} & 0.5\% & 67.15 & 61.5  & 9.8649E-04 & 1.0305E-02 & 1600  & 3     & 0.92785 & Adam \\
          &       & 1\%   & 73.47 & 69.9  & 1.4228E-04 & 1.3472E-02 & 100   & 13    & 0.68601 & Adam \\
          &       & 2\%   & 78.54 & 75.9  & 5.7111E-06 & 1.5544E-02 & 600   & 13    & 0.022622 & Adam \\
          &       & 3\%   & 79.97 & 78.5  & 4.0278E-05 & 2.7287E-02 & 4350  & 5     & 0.57173 & Adam \\
          &       & 4\%   & 81.49 & 80.4  & 1.4152E-05 & 2.3359E-02 & 2500  & 13    & 0.018578 & Adam \\
          &       & 5\%   & 81.82 & 81.7  & 1.2621E-03 & 1.5323E-02 & 3550  & 2     & 0.87352 & Adam \\
          & \multirow{6}[0]{*}{\textbf{CiteSeer}} & 0.5\% & 56.39 & 56.1  & 2.6983E-03 & 2.5370E-02 & 300   & 6     & 0.82964 & Adam \\
          &       & 1\%   & 65.04 & 62.1  & 1.6982E-03 & 1.5473E-02 & 2150  & 2     & 0.98611 & Adam \\
          &       & 2\%   & 69.48 & 68.6  & 9.7299E-05 & 4.9675E-02 & 2150  & 3     & 0.71216 & Adam \\
          &       & 3\%   & 71.09 & 70.3  & 1.7839E-04 & 3.0874E-02 & 2150  & 2     & 0.16549 & Adam \\
          &       & 4\%   & 72.32 & 70.8  & 5.6575E-05 & 3.5949E-02 & 4800  & 2     & 0.012576 & Adam \\
          &       & 5\%   & 72.8  & 71.3  & 2.8643E-04 & 1.6399E-02 & 2000  & 2     & 0.37308 & Adam \\
          & \multirow{4}[1]{*}{\textbf{Pubmed}} & 0.03\% & 62.94 & 62.2  & 1.2700E-03 & 1.4159E-03 & 128   & 4     & 0.76848 & RMSProp \\
          &       & 0.05\% & 68.31 & 68.3  & 1.1224E-03 & 9.9166E-05 & 256   & 3     & 0.85496 & RMSProp \\
          &       & 0.1\% & 73.29 & 72.7  & 6.0506E-04 & 1.0303E-03 & 256   & 2     & 0.97988 & RMSProp \\
          &       & 0.3\% & 79.63 & 79.2  & 1.1416E-03 & 6.1543E-04 & 128   & 1     & 0.989 & RMSProp \\
    \midrule
    \multirow{16}[2]{*}{\textbf{truncated Krylov}} & \multirow{6}[1]{*}{\textbf{Cora}} & 0.5\% & 72.96 & 61.5  & 3.3276E-03 & 1.0496E-04 & 128   & 18    & 0.76012 & RMSProp \\
          &       & 1\%   & 75.52 & 69.9  & 7.4797E-04 & 9.1736E-03 & 2048  & 20    & 0.98941 & RMSProp \\
          &       & 2\%   & 80.31 & 75.9  & 1.7894E-04 & 1.1079E-02 & 4096  & 16    & 0.97091 & RMSProp \\
          &       & 3\%   & 81.54 & 78.5  & 4.3837E-04 & 2.6958E-03 & 512   & 17    & 0.96643 & RMSProp \\
          &       & 4\%   & 82.47 & 80.4  & 3.6117E-03 & 4.1040E-04 & 64    & 25    & 0.021987 & RMSProp \\
          &       & 5\%   & 83.36 & 81.7  & 1.0294E-03 & 5.3882E-04 & 256   & 23    & 0.028392 & RMSProp \\
          & \multirow{6}[0]{*}{\textbf{CiteSeer}} & 0.5\% & 59.6  & 56.1  & 1.9790E-03 & 4.0283E-04 & 16    & 20    & 0.007761 & RMSProp \\
          &       & 1\%   & 65.95 & 62.1  & 7.8506E-04 & 8.2432E-03 & 64    & 24    & 0.28159 & RMSProp \\
          &       & 2\%   & 70.23 & 68.6  & 5.4517E-04 & 1.0818E-02 & 256   & 12    & 0.27027 & RMSProp \\
          &       & 3\%   & 71.81 & 70.3  & 1.4107E-04 & 5.0062E-03 & 1024  & 9     & 0.57823 & RMSProp \\
          &       & 4\%   & 72.36 & 70.8  & 4.8864E-06 & 1.8038E-02 & 4096  & 12    & 0.11164 & RMSProp \\
          &       & 5\%   & 72.24 & 71.3  & 2.1761E-03 & 1.1753E-02 & 5000  & 8     & 0.71473 & Adam \\
          & \multirow{4}[1]{*}{\textbf{Pubmed}} & 0.03\% & 69.07 & 62.2  & 6.8475E-04 & 2.8822E-02 & 4096  & 7     & 0.97245 & RMSProp \\
          &       & 0.05\% & 71.77 & 68.3  & 2.3342E+04 & 2.2189E-03 & 1024  & 8     & 0.93694 & RMSProp \\
          &       & 0.1\% & 76.07 & 72.7  & 4.2629E-04 & 4.1339E-03 & 2048  & 8     & 0.98914 & RMSProp \\
          &       & 0.3\% & 80.04 & 79.2  & 2.2602E-04 & 3.3626E-02 & 2000  & 7     & 0.070573 & Adam \\
    \bottomrule
    \bottomrule
    \end{tabular}%
  \label{tab:hyperparameters_without_validation}%
\end{table}%

\end{document}